\documentclass{article}

\usepackage{amsfonts}
\usepackage{amssymb}
\usepackage{amsmath}
\usepackage{amsthm}
\usepackage{color}
\usepackage{graphicx}
\usepackage{varwidth}
\usepackage{caption}
\usepackage{array}
\usepackage{subfig}
\usepackage{algorithm,algpseudocode}
\usepackage{mathtools}
\usepackage[inline]{enumitem}
\usepackage[export]{adjustbox}
\usepackage{verbatim}

\usepackage{enumitem}
\newlist{Properties}{enumerate}{2}
\setlist[Properties]{label=Property \arabic* ,leftmargin=*}

\def\IR{\relax{\rm I\kern-.18em R}}
\def\p{\partial}
\def\<{\langle}
\def\>{\rangle}

\DeclarePairedDelimiter\norm{\lVert}{\rVert}%
\makeatletter
\let\oldnorm\norm
\def\norm{\@ifstar{\oldnorm}{\oldnorm*}}
\makeatother

\def\sgn{\textrm{sgn}}
\def\({\right(}
\def\){\right)}

\algnewcommand{\Initialize}[1]{%
	\State \textbf{Initialize:}
	\Statex \hspace*{\algorithmicindent}\parbox[t]{.8\linewidth}{\raggedright #1}
}

\newtheorem{theorem}{Theorem}
\newtheorem{proposition}[theorem]{Proposition}
\newtheorem{definition}[theorem]{Definition}

\providecommand{\keywords}[1]{\textbf{\textit{Index terms---}} #1}

\begin{document}

	\title{Flows Generating Nonlinear Eigenfunctions
	}
	
	\author{Raz Z. Nossek
		\and
		Guy Gilboa
	}
	
	\maketitle

	\begin{abstract}
		Nonlinear variational methods have become very powerful tools for many image processing tasks. Recently a new line of research has emerged, dealing with nonlinear eigenfunctions induced by convex functionals. This has provided new insights and better theoretical understanding of convex regularization and introduced new processing methods. However, the theory of nonlinear eigenvalue problems is still at its infancy. We present a new flow that can generate nonlinear eigenfunctions of the form $T(u)=\lambda u$, where $T(u)$ is a nonlinear operator and $\lambda \in \mathbb{R} $ is the eigenvalue. We develop the theory where $T(u)$ is a subgradient element of a regularizing one-homogeneous functional, such as total-variation (TV) or total-generalized-variation (TGV). We introduce two flows: a forward flow and an inverse flow; for which the steady state solution is a nonlinear eigenfunction. The forward flow monotonically smooths the solution (with respect to the regularizer) and simultaneously increases the $L^2$ norm. The inverse flow has the opposite characteristics. For both flows, the steady state depends on the initial condition, thus different initial conditions yield different eigenfunctions. This enables a deeper investigation into the space of nonlinear eigenfunctions, allowing to produce numerically diverse examples, which may be unknown yet. In addition we suggest an indicator to measure the affinity of a function to an eigenfunction and relate it to pseudo-eigenfunctions in the linear case.
		
	\end{abstract}
	
	

	\markboth{Nossek and Gilboa}{Flows Generating Nonlinear Eigenfunctions}

	\section{Introduction}
	\label{sec:intro}
	
	Nonlinear convex functionals have become very instrumental in recent years in formulating mathematical solutions for a variety of image processing
	and computer vision problems, such as denoising \cite{rudin1992nonlinear, knoll2011second, louchet2011total, yang2013non, kindermann2005deblurring,Elmoataz2008Nonlocal, Gilboa2009Nonlocal}, optical flow \cite{motion2001Weickert, werlberger2011optical, deriche1995optical}, inpainting \cite{chan2001inpainting, burger2009cahn,dong2012wavelet}, 3D processing \cite{Schmaltz2012}, segmentation \cite{chan2001active, jung2012nonlocal, RegBasedSeg2013, MultiLabel2009convex, pock2009convex}	and more.
		
	These functionals are often used to regularize an inverse problem and direct the solution to be more probable and physical. This is done in order to cope with noisy, low quality or missing data. A very effective class of  functionals used in these cases is the family of one-homogenous functionals, which includes all norms and semi-norms. Specifically, functionals based on the $L^1$ norm of derivatives of the signal promote sparsity of the gradients and yield edge preservation, which is an essential characteristic in natural and medical imagery, motion fields, depth maps and other signals.
	
	The simplest, most practical and parameter-free one-homogeneous functional of this class is the total variation (TV), which is essentially the $L^1$ norm of the gradient, or more formally,
	\begin{equation}\label{eq:TV_def_gen}
	TV(u) = \sup \left \{ \int_{\Omega}u(x)\textnormal{div} \phi(x) \, dx: \phi \in \mathbf{C}_c^1(\Omega,\mathbb{R}^n), \|\phi\|_{L^{\infty}} \le 1 \right \}
	\end{equation}	
	where $\mathbf{C}_c^1(\Omega,\mathbb{R}^n)$ is the set of continuously differentiable vector functions of compact support in $\Omega$.
	Introduced in image processing by Rudin et al. \cite{rudin1992nonlinear} (known as the ROF model) for image denoising and deconvolution, this functional and its different variations were extensively used in various applications. Mathematically, a large body of theoretical research was devoted to explore its properties. For recent monographs on the subject see \cite{burger2013guide, chambolle2010introduction}.
	
	A more general and highly useful regularizer, proposed in recent years by Bredies et al. \cite{ bredies2010total, knoll2011second}, is the \emph{total generalized variation} (TGV), which is based on higher order derivatives and is defined in the following way,
	\begin{multline}\label{eq:TGV_def_gen}
	TGV_\alpha^k(u) = \sup \left\{ \int_{\Omega}u(x)\textnormal{div}^k v \, dx: v \in \mathbf{C}_c^k(\Omega,\textnormal{Sym}^k(\mathbb{R}^n)), \right. \\
	\left. \|\textnormal{div}^l\|_{L^{\infty}} \le \alpha_l, \,\, l=0,...,k-1 \vphantom{\int_t} \right\},
	\end{multline}
	where $\textnormal{Sym}^k(\mathbb{R}^n)$ denotes the space of symmetric tensors of order $k$ with arguments in $\mathbb{R}^n$, and $\alpha_l$ are fixed positive parameters. In this class, the second order form called TGV$^2_\alpha$ is practical and is able to cope well with discontinuities as well as linear transitions (with no staircasing effects, as induced by the TV functional). Preliminary analysis for TGV was performed in \cite{benning2013higher, poschl2013exact,papafitsoros2013study}. Another active field of research is formulating nonlocal and graph-based functionals  \cite{kindermann2005deblurring,louchet2011total,yang2013non, Elmoataz2008Nonlocal, Gilboa2009Nonlocal} which allow data-driven regularization with complex nonlocal interactions.
	
	As regularizers grow more complex, their theoretical analysis becomes extremely involved. In those cases one may need to resort to numerical solutions.
	A very significant characteristic of regularizers in image processing is the type of shapes which the regularizer can preserve within a variational minimization or a gradient descent flow.
	Nonlinear eigenfunctions belong to this class and are therefore very significant in a thorough study of regularizers.
	
	\subsection{Nonlinear eigenfunctions}
	There are several ways to generalize the linear eigenvalue problem $Lu =\lambda u$, where $L$ is a linear operator, to the nonlinear case (for some alternative ways see \cite{appell2004nonlinear}). We use the following formulation,
	\begin{equation}
	\label{eq:gen_ef_prob}
	T(u)=\lambda u,
	\end{equation}
	where $T(u)$ is a bounded nonlinear operator defined on an appropriate Banach space $\mathcal{U}$, and $\lambda \in \mathbb{R}$ is the eigenvalue (we restrict ourselves to the real-valued setting). In this paper we focus on the case of nonlinear eigenfunctions induced by a convex functional $J(u)$, where the subgradient element $p(u) \in \partial J(u)$ acts as a (possibly) nonlinear operator, with $\partial J(u)$ being the subdifferential of $J(u)$.
	Thus we focus on the following eigenvalue problem,
	\begin{equation}
	\label{eq:eigenfunction}
	p(u)=\lambda u, \,\, p(u) \in \partial J(u),
	\end{equation}
	where $u$ admitting \eqref{eq:gen_ef_prob} is an \emph{eigenfunction} and $\lambda \in \mathbb{R}$ is the corresponding \emph{eigenvalue}.
	Note that in some cases one restricts $u$ to have $\|u\|_{L^2}=1$, however in this paper we keep the un-normalized setting.
	
	For a proper, convex, one-homogeneous functional, a gradient flow is defined by
	\begin{equation}
	\label{eq:grad_flow}
	u_t=-p \,\,\,\, u|_{t=0}=f, \,\,\, p \in \partial J(u),
	\end{equation}
	where $u_t$ is the first time derivative of $u(t;x)$.
	It was shown in \cite{burger2016spectral} that when the flow is initialized with an eigenfunction (that is, $\lambda f \in  \partial J(f)$) the following solution is obtained:
	\begin{equation}
	\label{eq:grad_flow_sol}
	u(t;x) = (1-\lambda t)^+ f(x),
	\end{equation}
	where $(q)^+=q$ for $q>0$ and 0 otherwise. This means that the shape $f(x)$ is spatially preserved and changes only by contrast reduction throughout time.
	To avoid the reduction in contrast, techniques like inverse scale space \cite{burger2006nonlinear}, spectral filtering \cite{Gilboa2014SpecTV,burger2016spectral} or recent debiasing techniques \cite{Deledalle2015,brinkmann2016bias} can be used.
	
	A similar behavior (see \cite{burger2016spectral}) can be shown for a minimization problem with the $L^2$ norm, defined as follows:
	\begin{equation}
	\label{eq:ef_min_for}
	\underset{u} {\min} \,\, J(u) + \frac{\alpha}{2}\|f-u\|^2_{L^2}. 
	\end{equation}	
	In this case, when $f$ is an eigenfunction and $\alpha \in \mathbb{R}^+$ ($\mathbb{R}^+ = \{x \in \mathbb{R} \,| \, x \ge 0\}$) is fixed, the problem has the following solution:
	\begin{equation}
	\label{eq:ef_min_for_sol}
	u(x) = \left(1-\frac{\lambda}{\alpha}\right)^+f(x).
	\end{equation}
	In this case also, $u(x)$ preserves the spatial shape of $f(x)$ (as long as $\alpha > \lambda$).
	This was already observed by Meyer in \cite{Meyer[1]} for the case of a disk with $J$ the TV functional.
	We note that this also holds for quadratic regularizers with linear induced operators.
	This motivates us to explore eigenfunctions of different regularizers.
		
	Earlier research on nonlinear eigenfunctions induced by TV has been referred as \emph{calibrable sets}. First aspects of this line of research can be found in the work of Bellettini et al. \cite{bellettini2002total}.
	They introduced a family of convex bounded sets $C$ with finite perimeter in $\mathbb{R}^2$ that preserve their boundary throughout the TV flow (gradient flow \eqref{eq:grad_flow} where $J$ is TV). It is shown that the characteristic function $\chi_C$ with perimeter $P(C)$ which admits:
	\begin{equation}
	\underset{p \in \partial C}{\text{ess sup }} \kappa(p) \le \frac{P(C)}{|C|}
	\end{equation}
	is an eigenfunction, in the sense of \eqref{eq:gen_ef_prob}, where $u=\lambda_C \chi_C$ and
	\begin{equation}
	\label{eq:lam_tv}
	\lambda_C = \frac{P(C)}{|C|}.
	\end{equation}

	As discussed above, having a better understanding of properties of the eigenfunctions can assist in the choice of a proper functional for a given image processing task.
	The behavior of eigenfunctions under some kind of processing is illustrated in a toy example in  figure \ref{fig:EF_spec}. To explain this we first need to outline the TV spectral representation of \cite{Gilboa2014SpecTV}.
	
	\begin{figure}[!hp]
		\captionsetup{justification=centering}
		\centering
		\subfloat[ Numerical TV eigenfunction $g$]{\label{fig:TV_ef} \includegraphics[width=0.27\textwidth]{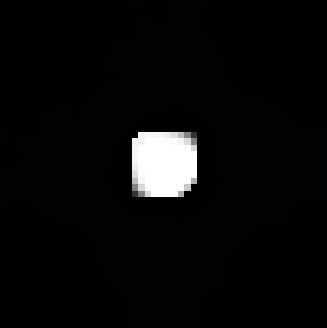}}
		\quad \qquad \qquad
		\subfloat[ Spectral response of $g$ ]{\label{fig:TV_ef_spec}
			\includegraphics[width=0.37\textwidth]{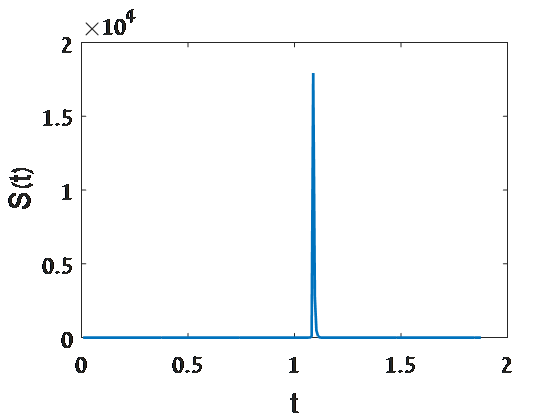}
		}
		\qquad \qquad \qquad
		\subfloat[ Eigenfunction with additive white Gaussian noise $n$ ($\sigma=0.3$), $f = g+n$. ]{\label{fig:TV_ef_noise}
			\includegraphics[width=0.27\textwidth]{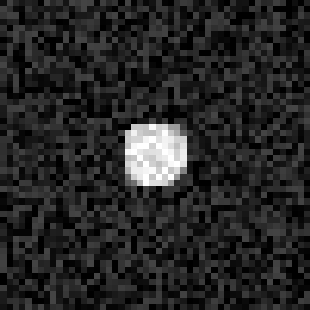}}
		\quad \qquad \qquad
		\subfloat[ Spectral response of $f$ ]{\label{fig:TV_ef_noise_spec}
			\includegraphics[width=0.37\textwidth]{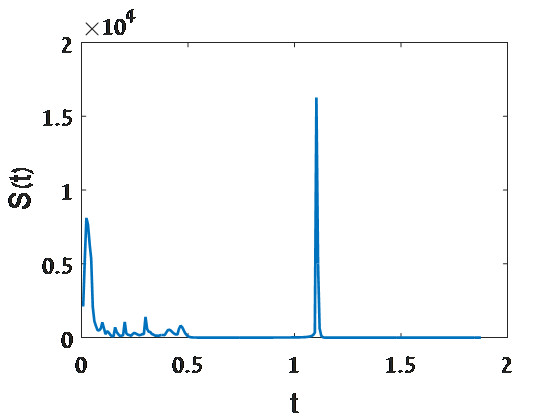}}
		\qquad \qquad \qquad
		\subfloat[Denoised $f$ using BM3D PSNR=24.66$_{dB}$]{\label{fig:TV_denoise_BM3D}
			\includegraphics[width=0.25\textwidth]{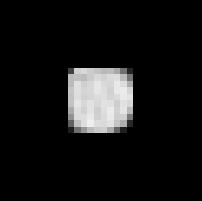}}
		\qquad
		\subfloat[Denoised $f$ using EPLL PSNR=24.62$_{dB}$]{\label{fig:TV_denoise_EPLL}
			\includegraphics[width=0.25\textwidth]{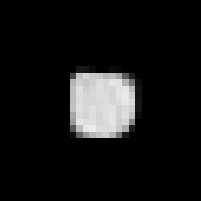}}
		\qquad
		\subfloat[Denoised $f$ using Spectral TV LPF PSNR=28.12$_{dB}$]{\label{fig:TV_denoise_spec}
			\includegraphics[width=0.25\textwidth]{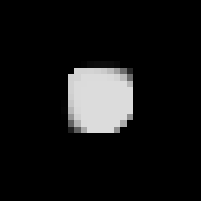}}
		\captionsetup{justification=justified}
		\caption{Example showing how a regularizer $J$ is very well suited to process an eigenfunction $g$ admitting $\lambda g \in \partial J(g)$.
				In this case $J$ is the (discrete) isotropic TV functional. From top left,
				(a) Eigenfunction $g$, (b) Its spectral response $S(t)$, (c) Eigenfunction with noise and its spectral response (d),
				performing denoising using: BM3D (e), EPLL (f) and TV-spectral filtering (g).}
		\label{fig:EF_spec}
	\end{figure}
	
	\subsection{Spectral TV}
	In \cite{Gilboa2014SpecTV} an alternative approach of spectral representation for TV was proposed. This was then generalized to one homogeneous functionals in \cite{burger2016spectral}. We briefly describe the basic TV setting. Let $f(x)\in BV$ be an input image with zero mean (for simplicity). The function $u(t;x)$ is the TV gradient descent solution, \eqref{eq:grad_flow}, with $J$ the TV functional. The TV transform is defined by
	\begin{equation}\label{eq:TV_trans}
	\phi(t;x)=u_{tt}(t;x)t,
	\end{equation}
	where $u_{tt}(t;x)$ is the second time derivative of $u(t;x)$.
		The function $\phi(t_0;x)$ is a spectral component of $f(x)$ at the scale $t_0$. It is shown in \cite{burger2016spectral}, that under a certain setting $\phi(t_0;x)$ is a difference of two eigenfunctions. Moreover, it admits an orthogonality property to all other $\phi$'s at different scales, $\langle \phi(t_0,x),\phi(t,x)\rangle=0$, $\forall t \ne t_0 $ .
	
	The reconstruction formula (inverse transform) is defined by,
	\begin{equation}
	\label{eq:inv_TV_trnas}
	f(x) = \int_{0}^{\infty}\phi(t;x) \, dt.
	\end{equation}
	Thus this representation can be interpreted as a nonlinear orthogonal decomposition of a signal into its multiscale components, based on a regularizing functional.
	Filtering in the spectral domain is performed by plugging a transfer function $H(t) \in \mathbb{R}$ (spectral filter) in the reconstruction formula,
	\begin{equation}
	\label{eq:TV_trans_filter}
	f_H(x) = \int_{0}^{\infty}\phi(t;x)H(t) \, dt.
	\end{equation}
	This procedure essentially attenuates, amplifies or preserves each spectral component.
	The spectrum $S(t)$ of the input signal $f(x)$ is defined in \cite{Gilboa2014SpecTV} by:
	\begin{equation}
	\label{eq:TV_trans_spec}
	S(t) =\|\phi(t;x)\|_{L^1}= \int_{\Omega}|\phi(t;x)| \, dx,
	\end{equation}
	with other variations suggested in \cite{burger2016spectral}.
	A significant property of the above representation is that when $f(x)$ is an eigenfunction with eigenvalue $\lambda$ (i.e admits \eqref{eq:gen_ef_prob}), the transform results in a single impulse at time $t=1/\lambda$ multiplied by $f(x)$, i.e.
	$$	\phi(t;x) = \delta (t-1/\lambda)f(x), $$
	where $\delta (\cdot)$ is the Dirac delta.
	
	In figure \ref{fig:EF_spec} an eigenfunction for the discrete TV functional is given as computed by the flow described later in section \ref{sec:flows} (note that contrary to the continuous case, the shape is not precisely convex and is not of constant value, as in the continuous case of \cite{bellettini2002total}). It can be seen in figure \ref{fig:TV_ef_spec}, that the spectral response $S(t)$ of the eigenfunction approaches a numerical delta. As this is based on a smoothing TV-flow, the noise response appears mostly in smaller scales and is well separated from the clean eigenfunction in the transform domain, figure \ref{fig:TV_ef_noise_spec}. Thus, in order to denoise one performs the nonlinear analog of an ideal low-pass-filter with $H(t)=1$ for $t \ge t_c$ and $0$ otherwise ($t_c$ is the cutoff scale, note here that high ``frequencies'' appear at low $t$). Denoising an eigenfunction is mostly suitable for such spectral filtering. As can be seen in figures \ref{fig:TV_denoise_BM3D}, \ref{fig:TV_denoise_EPLL}, and \ref{fig:TV_denoise_spec}, results compete well with state-of-the-art denoising algorithms such as BM3D \cite{dabov2007image} or EPLL \cite{zoran2011learning}.
	
	Therefore, by having a better understanding of the regularizer and its eigenfunctions, one can enhance the regularization quality by adapting the functionals to fit the class of signals to be processed.
	
	\subsection{Numerical Eigenvalue Algorithms}
	Linear eigenvalue problems arise in many fields of science and engineering: in civil-engineering they determine how resistant a bridge is to vibrations; in quantum mechanics they impose the modes of a quantum system and in fluid mechanics they induce the flow of liquids near obstacles.  Complex high dimensional eigenvalue problems arise today in disciplines such as machine learning, statistics, electrical networks and more. There is vast research and literature, accumulated throughout the years, on numerical solvers for linear eigenvalue problems \cite{wilkinson1965algebraic, trefethen1997numerical, Numerical2011Saad,borm2012numerical}.
	Given a matrix $A$, a common practice is to calculate an eigenvalue revealing factorization of $A$, where the eigenvalues appear as entries within the factors and the eigenvectors are columns in an orthogonal matrix used in the decomposition (e.g Schur factorization and unitary diagonalization). This is often performed by applying a series of transformations to $A$ in order to introduce zeros in certain matrix entries. This process is done iteratively until convergence. Notable algorithms applying such techniques are the \emph{QR algorithm} \cite{Francis1961QR} or the \emph{divide-and-conquer} algorithm \cite{Cuppen1980divide}. As a consequence, these methods are appropriate for linear operators on finite dimensional spaces (matrices), and unfortunately such techniques do not naturally extend to the nonlinear case. However, not all techniques perform a sequence of factorizations (or diagonalization).  One of such methods is the \emph{inverse power method} (IPM) and its extension, the \emph{Rayleigh quotient iteration} \cite{trefethen1997numerical}. Hein and B\"{u}hler \cite{Hein2010IPM} found a clever way to generalize the Rayleigh quotient iteration to the nonlinear eigenvalue problem case, with the same definition as in \eqref{eq:gen_ef_prob}.
	In section \ref{sec:prev_work} we describe this method in more details. In section \ref{sec:numerics} we compare our proposed flow to this state-of-the-art method.

	\subsection{Main Contributions}
	Our main contribution in this paper is presenting a new iterative flow-type method that can generate nonlinear eigenfunctions induced by convex one-homogeneous functionals.
	our contribution includes:
	\begin{enumerate}
		\item {Analyzing the properties of the flow, and showing it reaches a necessary condition for a steady-state if and only if $u(t)$ is an eigenfunction.}
		\item {Introducing a simple iterative scheme to advance the forward flow, which can use any modern convex solver that minimizes problems of the type $J(u)+\alpha\|f-u\|_{L^2}^2$.}
		\item {Performing several experiments for the cases of TV and TGV functionals and comparing the results to the state-of-the-art method of Hein and B\"{u}hler \cite{Hein2010IPM}. We show that our proposed method tends to find more complex eigenfunctions, with larger eigenvalues, and is less attracted to the simplest nontrivial eigenfunction (minimal positive eigenvalue) as often occurs in \cite{Hein2010IPM}. }
		\item {Presenting the possibility to use an inverse flow, especially directed for non-smooth high-eigenvalue cases and showing our method can be used in the linear case, under some assumptions on the linear operator $L$.}
		\item {Proposing a new measure of affinity for nonlinear eigenfunction, i.e. a measure which determines the proximity of a certain function to an eigenfunction of some nonlinear operator $T$. We also connect this to the notion of pseudo-eigenfunctions and pseudo-spectra in the linear case.}
	\end{enumerate}
	
	\section{Preliminaries}
	\label{sec:ef_prop}
	As this work aims at finding eigenfunctions numerically, it is more natural to be in a discrete setting.
	We assume a $d$ dimensional signal with $N$ pixels.
	We denote $\mathcal{X}$ as the Euclidean space $\mathbb{R}^N$ endowed with the $L^2$ inner product
	$\<u,v\> := \sum_{1\le i\le N}u_i v_i$ and the $L^2$ norm $\|u\| := \sqrt{\<u,u\>}$.
	
	\subsection{Properties of one-homogeneous functionals}
	
	Let $J(u)$ be a one homogeneous convex functional, that is
	\begin{equation}
	\label{eq:J1hom}
	J(\alpha u) = |\alpha|J(u),  \,\, \alpha \in \mathbb{R},
	\end{equation}
	and admits $J: \mathcal{X} \rightarrow \mathbb{R}$. Let $p$ belong to the subdifferential of $J(u)$:
	\begin{equation}
	\label{eq:subdif}
	\partial J(u) = \left\{p(u) \,|\, J(v)-J(u) \ge \< p(u),v-u \>, \forall v \in \mathcal{X}\right\}.
	\end{equation}
	We denote $p(u) \in \p J(u)$. $p$ also satisfies the relation induced by the Legender-Fenchel transform:
	\begin{equation}
	\label{eq:Legender}
	J^*(p):=\underset{u}{\sup}\left\{\<u,p\>-J(u)\right\}.
	\end{equation}
	And $J^*(p)$ is known as the dual functional (or convex conjugate \cite{ekeland1976convex}).
	
	For convex one homogeneous functionals it is well known \cite{ekeland1976convex} that:
	\begin{equation}
	\label{eq:up}
	J(u) =  \< u,p(u) \>, \forall p(u) \in \p J(u),
	\end{equation}
	and that
	\begin{equation}
	\label{eq:p_alpha}
	p(\alpha u) =  \sgn(\alpha) p(u),  \,\, \mathbb{R} \ni \alpha \ne 0.
	\end{equation}
	From \eqref{eq:subdif} and \eqref{eq:up} we have that a subdifferential of one-homogeneous functionals admits the following inequality:
	\begin{equation}
	\label{eq:subdif1hom}
	J(v) \ge \< p(u),v \>, \forall p(u) \in \p J(u), \,v \in \mathcal{X}.
	\end{equation}
	
	One-homogeneous functionals obey the triangle inequality:
	\begin{equation}
	\label{eq:triangle}
	J(u+v) \le J(u) + J(v).
	\end{equation}
	This can be shown by $J(u+v)=\<u+v,p(u+v)\>=\<u,p(u+v)\>+\<v,p(u+v)\>$
	and using \eqref{eq:subdif1hom} we have $J(u) \ge \<u,p(u+v)\>$ and $J(v) \ge \<v,p(u+v)\>$.
	
	By the Cauchy-Schwarz inequality Eq. \eqref{eq:up} also directly implies
	\begin{equation}
	\label{eq:J_CS}
	J(u) \le \|u\| \|p(u)\|, \forall p(u) \in \p J(u).
	\end{equation}
	
	The null space of a functional $J$ (which is a linear subspace, see e.g. \cite{benning2012ground}), is defined as
	\begin{equation}
	\mathcal{N}(J) = \{u \in \mathcal{X} \,\, | \,\, J(u)=0 \}.
	\end{equation}
	The orthogonal complement of the null space of $J$ (also a linear subspace) is
	\begin{equation}
	\label{eq:orth_null}
	\mathcal{N}(J)^{\perp} = \{v \in \mathcal{X} \,\, | \,\, \<v,u\>=0, \, \forall u \in \mathcal{N}(J) \}.
	\end{equation}

	We denote the projection operator onto $\mathcal{N}(J)$ by $P_0$ and the projection onto $\mathcal{N}(J)^{\perp}$ by $Q_0=I-P_0$.
	Note that for the TV case, projecting a function $f$ on $\mathcal{N}(J)^{\perp}$ can be done by enforcing $\langle f,1 \rangle=0$ or reducing the mean value of $f$.
	
	\subsubsection*{Basic properties of eigenvalues}
	
	One can generalize to the one-homogeneous case the relation of Eq. \eqref{eq:lam_tv} between $\lambda$ and the perimeter to area ratio which were given before in the specific case of a characteristic set, where $J$ is TV.
	For $J$ a one-homogeneous convex functional and $u$ an eigenfunction admitting
	\eqref{eq:eigenfunction} ($\|u\| > 0$) we have
	\begin{equation}
	\label{eq:lam_1hom}
	\lambda = \frac{J(u)}{\|u\|^2}.
	\end{equation}
	This can be easily shown by using \eqref{eq:up} and \eqref{eq:eigenfunction} having
	$$ J(u) = \<p(u),u \> = \<\lambda u, u \> = \lambda \| u \|^2.$$
		
	{\bf Condition for positive eigenvalues.}
	We now discuss briefly under what conditions eigenvalues are strictly positive $\lambda > 0$ for eigenfunctions of convex functionals.
	For the one homogeneous case this is a straightforward statement. Let $J$ be a convex positively one-homogeneous functional (therefore $J(u)\ge0$, $\forall u \in \mathcal{X}$).
	Then for any eigenfunction $u \notin \mathcal{N}(J)$,  that is $J(u)>0$, Eq.  \eqref{eq:lam_1hom} yields $\lambda>0$.
	We can have a broader statement in the case of general convex functionals:
	For $J$ a proper convex functional and $u$ an eigenfunction, if $J(u)>J(0)$ then $\lambda > 0$.
	This can be shown by using Eq. \eqref{eq:subdif} with $v=0$, yielding
	$$ J(0) - J(u) \ge \langle p(u),-u \rangle. $$
	For $p(u)=\lambda u$ we obtain $ J(u)-J(0) \le \lambda \| u\|^2$, thus
	$$ 	0 < \frac{J(u)-J(0)}{\|u\|^2} \le \lambda. $$
	
	\subsection{Previous work}
	\label{sec:prev_work}	
		
	We give here a brief overview of the method of Hein and B\"{u}hler \cite{Hein2010IPM}. The authors extend the inverse power method (for more information on the basic method see e.g \cite{Numerical2011Saad}) for finding eigenvalues and eigenvectors for matrices to the nonlinear case with one-homogeneous functionals. In order to understand the method in \cite{Hein2010IPM}, first let us consider the Rayleigh quotient that is defined as
	\begin{equation}
	\label{eq:Ray}
	F_{Rayleigh}(u) = \frac{\<u,Au\>}{\|u\|_2^2},
	\end{equation}
	where $A \in \mathbb{R}^{n \times n} $ is a real symmetric matrix and $u \in \mathbb{R}^n$.
	If $u$ is an eigenfunction of $A$ then $F_{Rayleigh}(u) = \lambda$ where $\lambda$ is the corresponding eigenvalue of $u$.
	In \cite{Hein2010IPM} the authors consider functionals $F$ of the form
	\begin{equation}
	\label{eq:IPM_func}
	F(u) = \frac{R(u)}{S(u)},
	\end{equation}
	where both $R$ and $S$ are convex and $R:\mathbb{R}^n \rightarrow \mathbb{R}^+, \,\,\, S:\mathbb{R}^n \rightarrow \mathbb{R}^+ $. One can observe that the functional in \eqref{eq:IPM_func} is a generalization of the functional in \eqref{eq:Ray}.
	A critical point $u^*$ of $F$ fulfills
	\begin{equation*}
	0 \in \partial R(u^*) - \lambda \partial S(u^*) ,
	\end{equation*}
	where $\partial R$, $\partial S$ are the subdifferentials of $R$ and $S$, respectively, and $\lambda = \frac{R(u^*)}{S(u^*)}$.
	We identify $R(u) = J(u)$ and $S(u) = \frac{1}{2}\|u\|^2_2$. Note that this equation now becomes the nonlinear eigenvalue problem \eqref{eq:eigenfunction}.
		
	The standard (linear) iterative IPM uses the scheme $Au^{k+1} = u^k$ in order to converge to the smallest eigenvector of $A$. This scheme can also be represented as an optimization problem:
	\begin{equation*}
	u^{k+1} = \underset{v}{\arg \min} \,\, \frac{1}{2} \<v,Av\> - \<v,u^k\>.
	\end{equation*}
	This can directly be generalized to the nonlinear case by
	\begin{equation}
	u^{k+1} = \underset{v}{\arg \min} \,\, J(v) - \<v,u^k\>.
	\end{equation}
	Specifically for one-homogeneous functionals a slight modification is required and the minimization problem is given by
	\begin{equation}
	u^{k+1} = \underset{\|v\| \le 1}{\arg \min} \,\, J(v) - \lambda^k\<v,u^k\>,
	\end{equation}
	i.e, adding the constraint that $\|v\| \le 1$ and the addition of $\lambda^k$, where $\lambda^k = \frac{J(u^k)}{\|u^k\|^2_2}$ to the minimization, in order to guarantee descent. 	
		
	\section{The Proposed Flows}
	\label{sec:flows}
	\subsection{Forward flow}
	\label{sec:for_flow}
	With sections \ref{sec:intro} and \ref{sec:ef_prop} outlining the background we can now introduce a method to obtain eigenfunctions. Let $J$ be a proper, convex, lower semi-continuous, one-homogeneous functional such that the gradient descent flow \eqref{eq:grad_flow} is well posed.
	We consider the following flow:
	\begin{equation}
	\label{eq:for_flow}
	u_t =  \frac{u}{\|u\|} - \frac{p}{\|p\|},\;\;\; p \in \p J(u),
	\end{equation}
	with $u|_{t=0}=f$, where $f$ admits $\|f\|\ne 0$,  $\langle f,1 \rangle=0$, $f \in \mathcal{N}(J)^{\perp}$. The later property can be achieved for any input $\tilde{f}$ by subtracting its projecting onto the null-space, $f = \tilde{f}-P_0 \tilde{f}$.
	Thus we have that $J(f) >0$. it can easily be shown that under these assumptions $\norm{u(t)} \ne 0$ and $\norm{p(t)} \ne 0$, $\forall t\ge0$, so the flow is well defined. We further assume that $J$ is a regularizing
	functional, invariant to a global constant change, such that
	$$ J(u)=J(u+c), \,\,\,\forall u \in \mathcal{X}, c \in \mathbb{R}. $$
	
	We will now show that this is a smoothing flow in term of the functional $J$ and an
	enhancing flow with respect the the $L^2$ norm, where a non-trivial steady state is reached
	for nonlinear eigenfunctions admitting Eq. \eqref{eq:eigenfunction} and only for them.
	
	\begin{theorem}
		\label{th:for_flow}
		The solution $u(t)$ of the flow of Eq. \eqref{eq:for_flow} has the following properties:
		
		\begin{Properties}
			\item The mean value of $u(t)$ is preserved throughout the flow: $$\< u(t),1 \>=0. $$
			\item \label{pro:J_des} $$\frac{d}{dt}J(u(t)) \le 0, $$ where equality is reached iff $u$ is an eigenfunction (admits \eqref{eq:eigenfunction}).
			\item \label{pro:u_asc} $$\frac{d}{dt}\|u(t)\|^2 \ge 0, $$ where equality is reached iff $u$ is an eigenfunction.
			\item A necessary condition for steady-state $u_t = 0$ holds iff $u$ is an eigenfunction.
		\end{Properties}
	\end{theorem}
	\vspace*{0.2 cm}
	\begin{proof}
		\begin{enumerate*}
		\item	From the invariance to constant change, $J(u)=J(u+c)$,
		\end{enumerate*}
		using \eqref{eq:Legender} it is easy to show that $$ J^*(p)=J^*(p) - \<c,p\>$$ yielding $c\<p,1 \>=0 $.		
		Let us define $Q(t) = \<u(t),1\>$. By using \eqref{eq:for_flow} and the above we obtain
		$$\frac{d}{dt}Q(t) = \< u_t(t),1\> = \< \frac{u}{\|u\|}-\frac{p}{\|p\|},1 \>  = \frac{1}{\|u\|}\<u,1\> = \frac{1}{\|u(t)\|}Q(t).$$
		The solution for this differential equation is given by $Q(t)=Be^{\int_0^t \frac{1}{\|u(\tau)\|} d\tau}$, where $B\in\mathbb{R}$ is some constant. Using the initial condition $u(t=0)=f$ and the fact that $\<f,1\> =0$ (hence $Q(t=0)=0$), yields $B=0$ resulting in $\<u(t),1\> = 0$, $\forall t \ge0$, i.e $u$ has mean zero and it is preserved throughout the flow.
		\begin{enumerate}
			\setcounter{enumi}{1}
			\item For the second claim we use \eqref{eq:eigenfunction} and \eqref{eq:up} obtaining
			$$ \frac{d}{dt}J(u(t)) = \langle p,u_t\rangle = \langle p, \frac{u}{\|u\|} - \frac{p}{\|p\|} \rangle = \frac{J(u)}{\|u\|}-\|p\| .$$
			Using \eqref{eq:J_CS} we conclude $\frac{J(u)}{\|u\|}-\|p\| \le 0 $ with equality if and only if $p$ is linearly dependent in $u$, hence an eigenfunction.\\
			\item The third claim can be verified in a similar manner by
			$$ \frac{d}{dt}\left(\frac{1}{2}\|u(t)\|^2\right) = \langle u,u_t\rangle = \langle u, \frac{u}{\|u\|} - \frac{p}{\|p\|} \rangle = \|u\| - \frac{J(u)}{\|p\|}.$$
			\item For the fourth claim, a necessary steady state condition is
			$$ u_t = \frac{u}{\|u\|} - \frac{p}{\|p\|} = 0.$$
			Therefore $p = \frac{\|p\|}{\|u\|}u$ and the eigenfunction equation \eqref{eq:eigenfunction} holds with $\lambda = \frac{\|p\|}{\|u\|}$.
			Naturally on the other direction, if \eqref{eq:eigenfunction} holds, $p = \lambda u$, we get $\frac{p}{\|p\|}=\frac{u}{\|u\|}$ and $u_t=0$.
		\end{enumerate}
	\end{proof}
	
	Notice that from \ref{pro:u_asc} of theorem \ref{th:for_flow} it might seem that $\|u\|^2_{L^2}$ can diverge. We show below that as long as the minimal nontrivial eigenvalue (with respect to the regularizer $J$ and the domain) is bounded from below by a positive constant, this does not happen.
	\begin{theorem}
		\label{th:L2_u_bounded}
		Let  $u(t)$ be the solution of the flow of Eq. \eqref{eq:for_flow}, then its $L^2$ norm is bounded from above.
	\end{theorem}
	\begin{proof}
		Let us define the minimal nontrivial eigenvalue for a specific value of the regularizer $J(u)=c>0$,  as
		$$ \lambda_{\min,c}: = \min_{u,\, \lambda u \in \partial J(u),\, J(u)=c} \lambda. $$
		Then when $\lambda_{\min,c}  > 0$ a bound on $\|u\|^2$ can be established.
		We examine the following optimization problem:
		\begin{equation*}
		\max \|u\|^2 \,\,\, \text{  s.t.   }  J(u)=c.
		\end{equation*}
		To solve this using Lagrange multipliers we define
		$$ \mathcal{L}(u,\alpha) = \|u\|^2 + \alpha (J(u)-c), $$
		yielding the necessary optimality condition,
		\begin{align*}
		\frac{\partial\mathcal{L}}{\partial u} &= 2u + \alpha p = 0, \\
		\frac{\partial\mathcal{L}}{\partial \alpha} &= J(u) - c = 0.
		\end{align*}
		Multiplying the first equation by $u$, summing and using $ J(u) = \<u,p\> $, we get $\alpha = -\frac{2\|u\|^2}{c} $ where  $p = -\frac{2}{\alpha} u $. Thus, the optimal $u$ is an eigenfunction with
		$\lambda =  -\frac{2}{\alpha} = \frac{c}{\|u\|^2} = \lambda_{\min,c}.$
		Moreover, for $c_2 > c_1$ we get $\lambda_{\min,c_2} < \lambda_{\min,c_1}$. This can be shown by choosing the minimal eigenfunction $u_{\min,c_1}$ corresponding to $\lambda_{\min,c_1}$ and multiplying it by $c_2/c_1$. Then this is clearly an eigenfunction restricted by $J(u)=c_2$ with a corresponding eigenvalue
		$$\lambda = \frac{J(u)}{\|u\|^2}=\frac{c_2}{\|u_{\min,c_1}c_2/c_1\|^2}= \frac{c_1}{c_2}\lambda_{\min,c_1} < \lambda_{\min,c_1}.$$
		
		Using the fact that $J(u(t))$ of the flow is decreasing with time we have $c \le J(f)$
		which yields the bound
		$$  \|u(t)\|^2|_{J(u(t))=c} \le \max_{J(u)=c} \|u\|^2 = \frac{c}{\lambda_{\min,c}} \le \frac{J(f)}{\lambda_{\min,J(f)}}, \;\; \forall t\ge 0. $$

		We remind that $f \in \mathcal{N}(J)^{\perp}$. 
		It is shown in \cite{burger2016spectral} Lemma 4 that if $p \in \partial J(u)$ then $p \in \mathcal{N}(J)^{\perp}$. Therefore, since our flow is a linear combination of $u$ and $p$ we are kept in the subspace $\mathcal{N}(J)^{\perp}$ and $J(u(t)) > 0$, $\forall t\ge0$.
	\end{proof}
	
	Another remark is that this process often does not converge to the eigenfunction with the smallest eigenvalue, and depends on the initialization of $f$.
	Note that from the above we can observe another interesting property of $\lambda$. As $\norm{(u(t))}$ is increasing with time and $J(u(t))$ is decreasing, then when an eigenfunction is reached, its eigenvalue $\lambda$ is bounded by
	\begin{equation}
	\label{eq:lam_dom}
		0 < \lambda \le \frac{J(f)}{\norm{f}^2}.
	\end{equation}
	
	\begin{figure}[!ht]
		\centering
		\subfloat[]{\label{fig:vec_unit_circ}
			\includegraphics[width=0.4\textwidth]{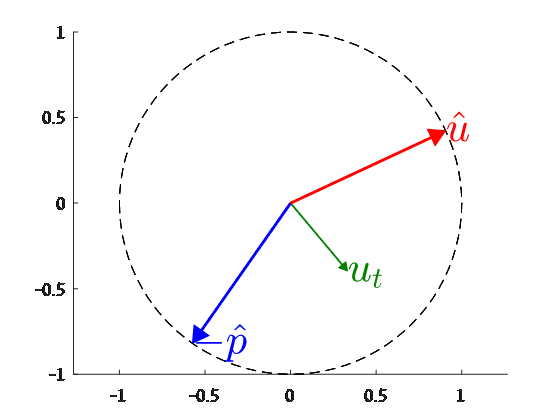}
		}
		\subfloat[]{\label{fig:vec_unit_circ_ef}
			\includegraphics[width=0.4\textwidth]{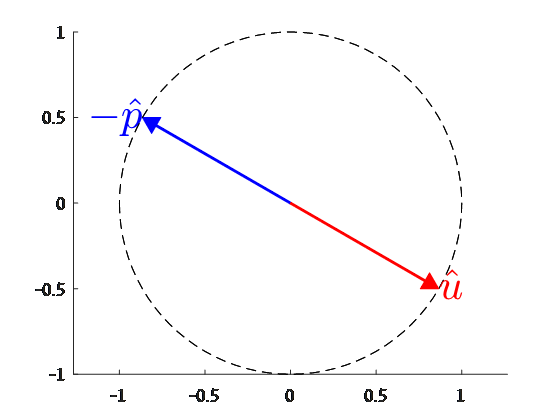}
		}
		\caption{An illustration of the geometric interpretation of the flow and the relation between $\hat{u}$ and $\hat{p}$. Figure (a) illustrates the general case where $u$ is not an eigenfunction induced by $J(u)$, while figure (b) illustrates the case where $u$ is an eigenfunction. Note that for this case $\hat{u}$ and $\hat{p}$ are exactly opposite one to another, yielding thus $u_t=0$.}
		\label{fig:unit_circ_interp}
	\end{figure}
	
	\subsubsection{Interpretation and regularity}
	One can define the $L^2$ unit vectors in the directions $u$ and $p$, respectively, as
	$$ \hat{u}=\frac{u}{\|u\|},\;\; \hat{p}=\frac{p}{\|p\|}, $$
	with $p \in \p J(u)$.
	The flow \eqref{eq:for_flow} can be rewritten as
	$$ 	u_t =  \hat{u} - \hat{p}. $$
	Thus there are two competing unit vectors. Notice that for one-homogeneous functionals $\<u,p\> = J(u) > 0$, and therefore the angle between $u$ and $p$ is in the range $(-\frac{1}{2}\pi,\frac{1}{2}\pi)$. Using this observation we later define an indicator which measures how close a function is to be an eigenfunction, see section \ref{sec:aff_meas}. The absolute angle between $\hat{u}$ and $-\hat{p}$ is larger than $\frac{\pi}{2}$, see figure \ref{fig:vec_unit_circ}, where for an eigenfunction $\hat{u}$ and $-\hat{p}$ are exactly at opposite directions (angle $\pi$) canceling each others contribution to the flow, enabling a steady-state solution (figure \ref{fig:vec_unit_circ_ef}).
	
	Regarding regularity, the flow \eqref{eq:for_flow} is essentially a time rescale of the gradient flow \eqref{eq:grad_flow} with amplification of $u$, so as long as there is no blow-up in $u$, the signal
	becomes smoother in terms of $J$ and regularity is maintained.
	
	\subsection{Inverse flow}
	An alternative flow which works in the inverse direction of \eqref{eq:for_flow} can also be defined:
	\begin{equation}
	\label{eq:inv_flow}
	u_t =  -\frac{u}{\|u\|} + \frac{p}{\|p\|},\,\,\, p \in \p_u J(u),
	\end{equation}
	with $u|_{t=0}=f$.
	
	This is an anti-smoothing flow in term of the functional $J$ and a reducing
	flow with respect the the $L^2$ norm, where also here a necessary steady state condition is reached
	for nonlinear eigenfunctions admitting Eq. \eqref{eq:eigenfunction} and only for them.
	
	\begin{theorem}
		\label{th:inv_flow}
		The solution $u(t)$ of the flow of Eq. \eqref{eq:inv_flow} has the following properties:
		\begin{enumerate}
			\item $$\frac{d}{dt}J(u(t)) \ge 0 $$ where equality is reached iff $u$ is an eigenfunction.
			\item $$\frac{d}{dt}\|u(t)\|^2 \le 0 $$ where equality is reached iff $u$ is an eigenfunction.
			\item A necessary condition for steady-state $u_t = 0$ holds iff $u$ is an eigenfunction.
		\end{enumerate}
	\end{theorem}
	\begin{proof}
		The proof follows the same lines as the one of theorem \ref{th:for_flow}.
	\end{proof}
	
	From preliminary experiments, this flow tends to produce non-smooth eigenfunctions with large eigenvalues, as can be expected.
	We point out this formulation, however in this paper this direction is not further developed.

	\section{Extension to the linear case}
	Although the flow was developed for nonlinear convex functionals, under some constraints the method works for linear operators as well. First, let us rewrite the forward flow \eqref{eq:for_flow} for some linear operator $L: \mathcal{V} \rightarrow \mathcal{V}$ over $\mathbb{R}$,
	\begin{equation}\label{eq:liner_for_flow}
	u_t = \frac{u}{\|u\|} - \frac{Lu}{\|Lu\|},
	\end{equation}
$u|_{t=0}=f$, $\<f,1\>=0$, and $f$ not an element in the null space of $L$. Here we seek to find a function $u$ which is a linear eigenfunction, $Lu=\lambda u$. We would like to keep a similar framework as in the nonlinear case, and therefore assume that $L$ is a positive-semidefinite operator, i.e. $\forall u\in\mathcal{V} , \,\, \<Lu,u\> \ge 0$ (as a consequence $L$ is a self-adjoint operator). Another assumption is that for a constant $c \in \mathbb{R}, \,\, L(cI) = 0$, where $I$ is the identity, or $L(u+cI) = Lu$.	
Within the above setting, one obtains a flow with similar properties as in theorem \ref{th:for_flow}.
	\begin{proposition}
		\label{th:linear_flow}
		The solution $u(t)$ of the flow of Eq. \eqref{eq:liner_for_flow} has the following properties:
		\begin{Properties}
			\item The mean value of $u(t)$ is preserved throughout the flow: $$\< u(t),1\> = 0, $$
			\item $$\frac{d}{dt}\<Lu,u\> \le 0, $$ where equality is reached iff $u$ is an eigenfunction. 
			\item $$\frac{d}{dt}\|u(t)\|^2 \ge 0, $$ where equality is reached iff $u$ is an eigenfunction.
			\item A necessary condition for steady-state $u_t = 0$ holds iff $u$ is an eigenfunction.
		\end{Properties}
	\end{proposition}
	\vspace*{0.2 cm}
	\begin{proof}
The proof follows similar arguments as for the one-homogeneous case. For the first property we use the fact that $ L(cI) = 0$, thus $ \<cI,Lu\> = 0 $ and one can show
the zero mean is preserved throughout the flow. The second property is shown by deriving the expression $\<Lu,u\>$ in time and plugging for $u_t$ the identity of \eqref{eq:liner_for_flow}. The third property uses Cauchy-Schwarz by
$$\frac{d}{dt}\left(\frac{1}{2}\|u\|^2\right) = \<u,u_t\> = \<u,\frac{u}{\|u\|} - \frac{Lu}{\|Lu\|}\> = \|u\| - \frac{\<Lu,u\>}{\|Lu\|} \ge 0.$$
The fourth property is straightforward for linear eigenfunctions.
\end{proof}
	
	\section{Pseudo-Eigenfunctions}
	The first introduction to the idea of \emph{pseudospectra} was given by Landau \cite{Landau1975}, who used the term $\varepsilon$-spectrum. Further extension of the topic was given in \cite{Varah1979Separation,chatelin1981spectral}, generalizing the theory for matrices and linear operators. Trefethen coined the term \emph{pseudospectra} \cite{trefethen1990approximation, trefethen1991pseudospectra} presenting an overview
	of the theory and applications in \cite{trefethen2005spectra}.

 	Given two linear operators $L$ and $E$, a pseudo-eigenfunction $u$ of $L$ admits the following eigenvalue problem
	\begin{equation}
	\label{eq:pseudo-ef}
	(L+E)u = \lambda u,\,\,\, \textrm{s.t.} \,\, \|E\| \le \varepsilon.
	\end{equation}
	That is, $u$ is an eigenfunction of an operator which is very similar to $L$, up to a small perturbation.
 The corresponding value $\lambda$ is said to be a \emph{pseudo-eigenvalue}, or more precisely an element in the $\varepsilon$-\emph{pseudosepctra} of $L$.
 Note that $\lambda$ does not have to be close to any eigenvalue of $L$, but is an exact eigenvalue of some operator similar to $L$.

	For nonlinear operators, it is not trivial how this notion could be generalized (as two operators cannot simply be added).
	Therefore, we define a somewhat different notion, which we refer to as a \emph{measure of affinity to eigenfunctions}.
	The measure is in the range $[0,1]$ and attains a value of 1 for eigenfunctions (and only for them). When it is very close to 1, this can be considered as an alternative definition of a pseudo-eigenfunction, which is a very useful notion in the discrete and graph case, as one may not be able to obtain a precise nonlinear eigenfunction in all cases (but may reach numerically a good approximation). We show below the exact relation for the linear case.
	
	\subsection{Measure of affinity of nonlinear eigenfunctions}
	\label{sec:aff_meas}
	Let $T$ be a general nonlinear operator in a Banach space $\mathcal{X}$,  $T:\mathcal{X} \to \mathcal{X}$ embedded in $L^2$ such that $T(u) \in L^2$.
	The corresponding nonlinear eigenvalue problem is \eqref{eq:gen_ef_prob}, ($T(u) = \lambda u$).
	
	\begin{definition}
		\label{def:A_meas}
		The measure $\mathcal{A}_{T}(u)$ of the affinity of a function $u$ to an eigenfunction, based on the operator $T$, with $\| u \|\ne 0$, $\| T(u) \|\ne 0$ , is defined by
		\begin{equation}
		\label{eq:ef_aff}
		\mathcal{A}_{T}(u):=\frac{|\langle u,T(u) \rangle|}{\| u \| \cdot \|T(u) \|}.
		\end{equation}
	\end{definition}
	
	\begin{proposition}
		$\mathcal{A}_{T}(u)$ admits the following
		\begin{equation}\label{eq:A_bound}
		0 \le \mathcal{A}_T(u) \le 1, \,\, \mathcal{A}_T(u)=1 \,\, \text{iff u admits the eigenvalue problem}.
		\end{equation}
	\end{proposition}
	
	\begin{proof}
		This is an immediate consequence of the Cauchy-Schwarz inequality.
	\end{proof}
That is, the measure is 1 for all eigenfunctions and only for them (we remind that for the Cauchy-Schwarz inequality equality is attained if and only if the two functions are linearly dependent). The measure then has a graceful degradation from 1 to 0.

	Let us define the projection of $u$ onto the plane orthogonal to $T(u)$:
	$$w:=  u - \frac{\langle u,T(u) \rangle}{\|T(u)\|^2}T(u).$$
	Then $\mathcal{A}_{T}(u)$ decreases as $\|w\|$ increases, where for eigenfunctions $\|w\|=0$.	
	Using the above we determine a pseudo-eigenfunction being close up to $\varepsilon$ to an exact eigenfunction of a nonlinear operator, if the following bound on $\mathcal{A}_T(u)$ holds
\begin{equation}
	\label{eq:A_eps}
	\mathcal{A}_T(u) \ge 1- \varepsilon.
	\end{equation}
	
	\qquad \qquad \qquad
	\begin{figure}[!ht]
		\centering
		\subfloat[]{
			\includegraphics[width=0.35\textwidth]{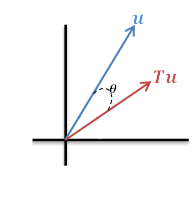}
		}
		\label{fig:geom_th_eps_clr}
		\qquad \qquad \qquad
		\subfloat[]{
			\includegraphics[width=0.35\textwidth]{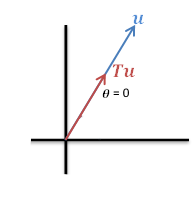}
		}
		\label{fig:geom_th_eps_ef_clr}
		\caption{An illustration of the angle induced by $u$ and $T(u)$. Figure (a) shows the case induced when $u$ is an arbitrary function, while figure (b) illustrates the case that $u$ is an eigenfunction of $T$.}
		\label{fig:geom_interp}
	\end{figure}
	
{\bf Geometric interpretation of the measure.}

Considering definition \ref{def:A_meas}, it can be written as  $\mathcal{A}_{T}(u)=\cos(\theta)$, i.e $\mathcal{A}_{T}(u)$ is based on the angle between $u$ and $T(u)$.
Thus it may be more insightful to look at $\theta$ itself,
	\begin{equation}
	\label{eq:theta_measure}
	\theta = \cos^{-1}(\mathcal{A}_{T}(u)).
	\end{equation}
An illustration of two cases, non-eigenfunction (a) and eigenfunction (b), is shown in  figure \ref{fig:geom_interp}.
Both values of $\mathcal{A}_{T}(u(t))$ and $\theta(u(t))$ were computed as a function of time throughout several flows and are shown in the experimental section.

		\begin{figure}[t]
			\captionsetup{justification=centering}
			\centering
			\subfloat[input function]{
				\includegraphics[width=0.20\textwidth,valign=c]{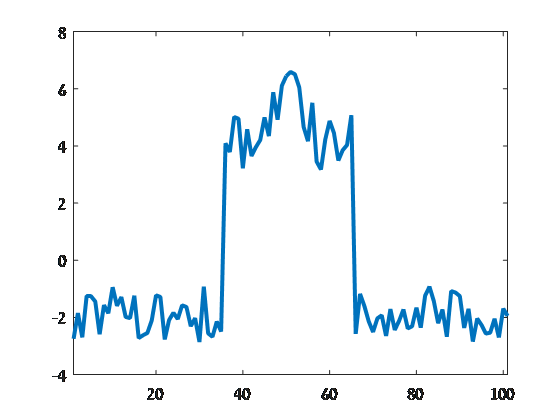}
			}
			\label{fig:tv_point_input}
			\begin{minipage}{0.75\textwidth}
				\captionsetup{justification=centering}
				\subfloat[intermidiate step in proposed method]{
					\includegraphics[width=0.26\textwidth,valign=b]{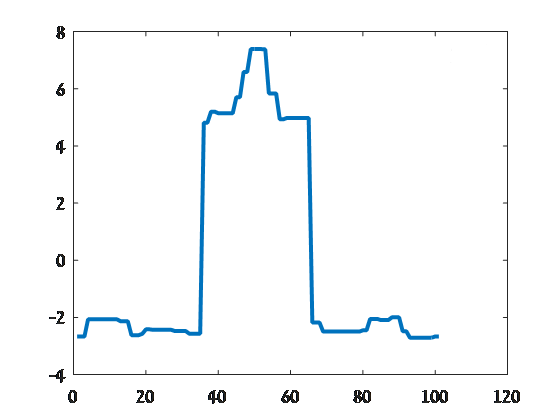}
				}
				\label{figtv_point_first_inst}
				\subfloat[intermidiate step in proposed method]{
					\includegraphics[width=0.26\textwidth,valign=b]{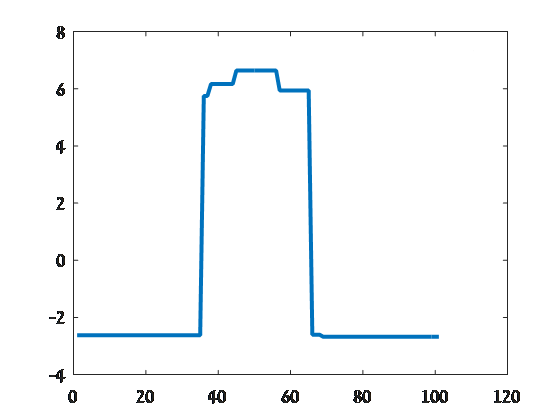}
				}
				\label{fig:tv_point_final_inst}
				\subfloat[converged E.F $\lambda_{prop} = 0.436$]{
					\includegraphics[width=0.23\textwidth,valign=b]{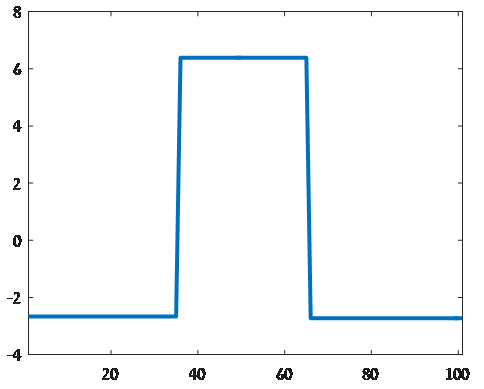}
				}
				\\
				\subfloat[intermediate step in IPM]{
					\includegraphics[width=0.26\textwidth,valign=b]{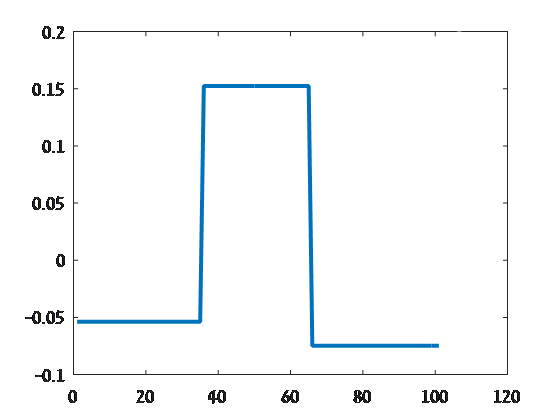}
				}
				\label{fig:ttv_point_ipm_first_inst}
				\subfloat[intermediate step in IPM]{
					\includegraphics[width=0.26\textwidth,valign=b]{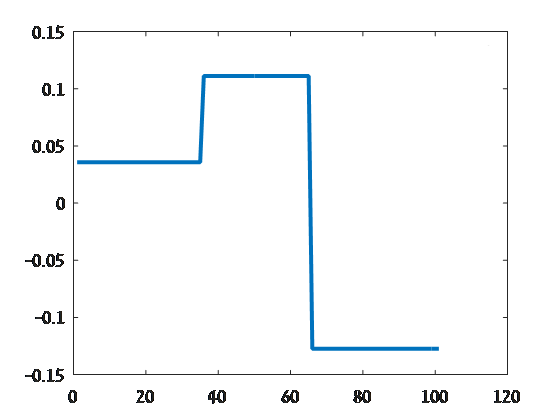}
				}
				\label{fig:tv_point_ipm_final_inst}
				\captionsetup{justification=centering}
				\subfloat[converged E.F $\lambda_{IPM} = 0.208$]{
					\includegraphics[width=0.23\textwidth,valign=b]{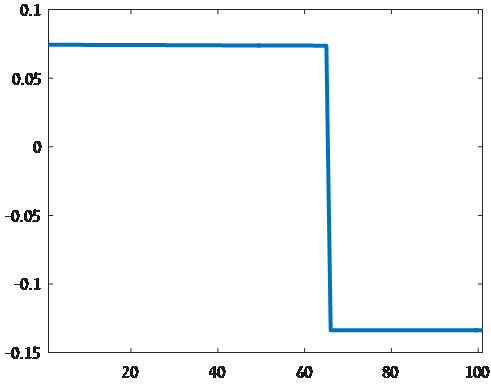}
				}
				
			\end{minipage}
			
			\captionsetup{justification=justified}
			\caption{A 1D example of the two methods for the TV functional. The upper row represents the proposed method, while the lower row is the IPM. (a) is the initial input. (b-c) \& (e-f) are examples of intermediate steps of the iterative methods. (d) \& (g) shows the final state (i.e the eigenfunction) each method converged to.}
			\label{fig:f_point_noise}
		\end{figure}

    \subsubsection{The 1-homogeneous and linear case}
	For eigenfunctions induced by one-homogeneous functionals we give
	the following adaptation of \eqref{eq:ef_aff}.
	
	\begin{equation}
	\label{eq:ef_aff_J}
	\textrm{A}_{p(u)}(u)  =  \frac{J(u)}{\| p(u)\| \cdot \|u \|} = \frac{\langle p(u),u\rangle}{\| p(u)\| \cdot \|u \|} ,
	\end{equation}
	for all $p(u) \in \partial J(u)$ (and having $J(u)\ge0$ the absolute expression in the numerator can be omitted).

	Having a linear operator $L$ the definition is now given by
	\begin{equation}
	\textrm{A}_{L}(u) =  \frac{\langle Lu,u\rangle}{\| Lu\| \cdot \|u \|} .
	\end{equation}
	We would like to show a connection between $A_L(u)$ and the pseudo-eigenfunction definition given in \eqref{eq:pseudo-ef}. Let $u$ admit \eqref{eq:pseudo-ef}, then
	\begingroup
	\addtolength{\jot}{1em}
	\begin{align*}
	1 = \textrm{A}_{L+E}(u) &=  \frac{\langle (L+E)u,u\rangle}{\| (L+E)u\| \cdot \|u \|} = \frac{\langle Lu,u\rangle}{\| (L+E)u\| \cdot \|u \|} + \frac{\langle Eu,u\rangle}{\| (L+E)u\| \cdot \|u \|} \\
	& \le \frac{\langle Lu,u\rangle}{(\norm{Lu} + \norm{Eu}) \cdot \|u \|} + \frac{ \norm{Eu} \cdot \norm{u}} {\| (L+E)u\| \cdot \|u \|}, \\
	\intertext{where for the first expression we use the triangle inequality in the denominator and for the second expression the Cauchy-Schwarz inequality in the numerator. Then, using $(L+E)u=\lambda u$, $\norm{E}\le \varepsilon$ and the definition of an operator norm  we get}
	& \le \frac{\langle Lu,u\rangle}{\norm{Lu} \cdot \|u \|} + \frac{ \norm{E} \cdot \norm{u}} {\| (L+E)u\|} \le \textrm{A}_{L}(u) + \frac{\varepsilon}{\lambda},
	\end{align*}
	\endgroup
	and we conclude that
	\begin{equation}
	\textrm{A}_{L}(u) \ge 1 - \frac{\varepsilon}{\lambda}.
	\end{equation}

	\section{Results}
	\label{sec:numerics}
	In the following section we present numerical results for our algorithm. We show results for the TV and TGV functionals, and visualize the geometric interpetation of our new measure. Further more, we compare our results to another technique by Hein and B\"{u}hler \cite{Hein2010IPM} as described in section \ref{sec:prev_work}.
	
	\subsection{Discretization}
	\label{sec:discrete}
	For the purpose of implementing numerically the methods presented in this paper we use Chambolle and Pock's primal-dual algorithm \cite{Chambolle2011} for solving the optimization problems defined for each method and each functional (TV and TGV). As the chosen discretization can affect the solution and the results at convergence (numerical eigenfunctions) we specify the precise gradient and divergence operators used in these experiments.
We use the standard first order forward/backward-difference operators which are commonly used for TV and TGV (see e.g.\cite{Chambolle[1],agco06}).
For $ u \in \mathcal{X}$ the gradient $\nabla u$ is a vector $\in \mathcal{X} \times \mathcal{X}$ given by: $(\nabla u)_{i,j} = ((\nabla u)^1_{i,j}, (\nabla u)^2_{i,j})$, with
	\[
	(\nabla u)^1_{i,j} =
	\begin{cases}
	u_{i+1,j} - u_{i,j}, & \text{if } i<N \\
	0, & \text{if } i=N
	\end{cases}
	\]
	and
	\[
	(\nabla u)^2_{i,j} =
	\begin{cases}
	u_{i,j+1} - u_{i,j}, & \text{if } j<N \\
	0, & \text{if } j=N
	\end{cases}
	\,\,\, .
	\]
	\vspace{5mm}
	
\begin{figure}[t]
		\centering
		\includegraphics[width=0.95 \textwidth,valign=c]{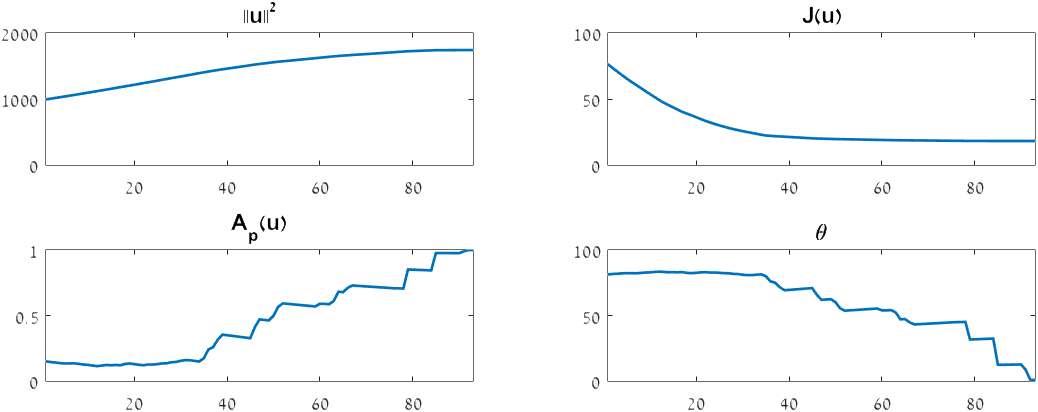}
		\caption{The evolution of $J(u),\, \|u\|^2, \, A_p(u)$ and $\theta$ as a function of $t$, for the given initial input in figure \ref{fig:f_point_noise}. Notice that $\|u\|^2$ is monotonically increasing and that $J(u)$ is monotonically decreasing.}
		\label{fig:tv_point_outcome}
	\end{figure}

The discrete divergence is the adjoint of the discrete gradient, defined by
	\begin{align*}
	(\text{div} \, z)_{i,j} &=
	\begin{cases}
	z^1_{i,j} - z^1_{i-1,j}, & \text{if } 1<i<N, \\
	z^1_{i,j}, & \text{if } i=1, \\
	-z^1_{i-1,j}, & \text{if } i=N,
	\end{cases} \\
	& +
	\begin{cases}
	z^2_{i,j} - z^2_{i,j-1}, & \text{if } 1<j<N, \\
	z^2_{i,j}, & \text{if } j=1, \\
	-z^1_{i,j-1}, & \text{if } j=N.
	\end{cases}
	\end{align*}
Other discretizations (such as spatially symmetric ones) would yield different eigenfunctions.	
	
	\subsection{Numerical implementation}
	\label{sec:examples}
Recall the basic forward flow given in \eqref{eq:for_flow}. Rewriting the PDE in a discrete semi-implicit setting yields
	\begin{equation}
	u^{k+1} = u^k + \Delta t \left( \frac{u^{k+1}}{\|u^k\|} - \frac{p^{k+1}}{\|p^k\|} \right),
	\end{equation}
	with $\Delta t$ indicating the chosen time-step to use. This equation can be reformulated into the following optimization problem
	\begin{equation}
	\label{eq:flow_opt_rep}
	u^{k+1} = \underset{v}{\arg \min} \,\, \left\{J(v) + \frac{\|p^k\|}{2\Delta t} \left(1 - \frac{\Delta t}{\|u^k\|} \right) \norm{\frac{u^k}{1 - \frac{\Delta t}{\|u^k\|}} - v}^2_{L^2} \right\},
	\end{equation}
	where $p \in \partial J(u) $.
We solve this optimization problem iteratively until convergence. Our stopping criterion is based on the affinity measure as defined in \eqref{eq:theta_measure}, when the difference between consecutive steps is smaller than a predefined threshold $\epsilon$ as shown in figure \ref{alg:ef_gen}. The algorithm consists of solving a non-smooth convex optimization problem for which several numerical algorithms are suitable. We chose as a solver a first order primal-dual algorithm \cite{Chambolle2011}. It turns out this solvers is well fit for these kind of problems and we also readily get $p$. The following values were used in all experiments: $\Delta t = 0.2$,  $\epsilon = 0.1$, and $\theta_{thresh} = 1$.
	
	\begin{algorithm}
		\caption{Computing a nonlinear eigenfunction for a convex one-homogeneous functional}
		\label{alg:ef_gen}
		\begin{algorithmic}[1]
			\Initialize {$u^0 = f, \,\, \Delta t, \,\, \epsilon,\,\, \theta_{thresh}$, Compute $p^0=p\in J(f)$.}
			\Repeat
			\State{$u^{k+1} = \underset{v}{\arg \min} \,\, \left\{J(v) + \frac{\|p^k\|}{2\Delta t} \left(1 - \frac{\Delta t}{\|u^k\|} \right) \norm{\frac{u^k}{1 - \frac{\Delta t}{\|u^k\|}} - v}^2_{L^2} \right\}, \,\,\,\,\, p^{k+1} \in \partial J(u^{k+1}) $ }
			\vspace{2mm}
			\State{ $\textrm{A}^{k+1}_{p}(u) = \frac{\<u^{k+1},p^{k+1}\>}{ \norm{p^{k+1}} \cdot \norm{u^{k+1}} }$ }
			\vspace{2mm}
			\State{ $\theta^{k+1} = \cos^{-1}(\textrm{A}^{k+1}_{p}(u)) $}
			\vspace{2mm}
			\Until {$ |\theta^{k+1} - \theta^{k} | < \epsilon $} \textbf{and} $\theta^{k+1} \le \theta_{thresh}$ \\
			\Return $u^{k+1} $
		\end{algorithmic}
	\end{algorithm}
		
The inverse flow given in equation \eqref{eq:inv_flow} can not be reformulated as an optimization problem as it is not guaranteed that the problem is convex. Therefore in order to implement the inverse flow we utilize an explicit scheme. We write \eqref{eq:inv_flow} in an explicit discrete setting as follows
\begin{equation}\label{eq:inv_explicit}
	u^{k+1} = u^k + \Delta t \left( -\frac{u^{k}}{\|u^k\|} + \frac{p^{k}}{\|p^k\|} \right).
\end{equation}
The algorithm to find eigenfunctions using the inverse flow is the same as in algorithm \ref{alg:ef_gen}, but with a slight change. We replace the optimization problem in line 3 with the explicit equation given in \eqref{eq:inv_explicit}. $p^k$ can still be computed as the subgradient of $J(u^k)$. All other parts of the algorithm remain the same. An example between the different results produced by the two flows: forward and inverse; for the same input is given in figure \ref{fig:tv_for_inv}.

		\begin{figure}[!ht]
			\captionsetup{justification=centering}
			\centering
			\subfloat[input function]{
				\includegraphics[width=0.2\textwidth,valign=c]{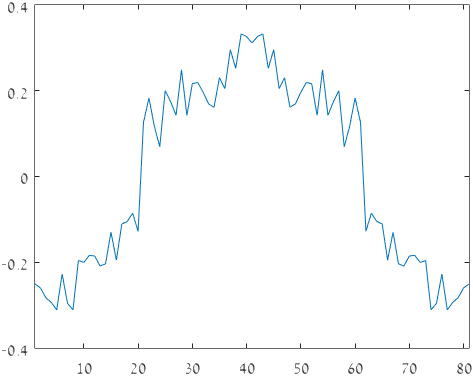}
			}
			\label{fig:tgv_ef_input}
			\begin{minipage}{0.75\textwidth}
				\captionsetup{justification=centering}
				\subfloat[intermidiate step in proposed method]{
					\includegraphics[width=0.25\textwidth,valign=b]{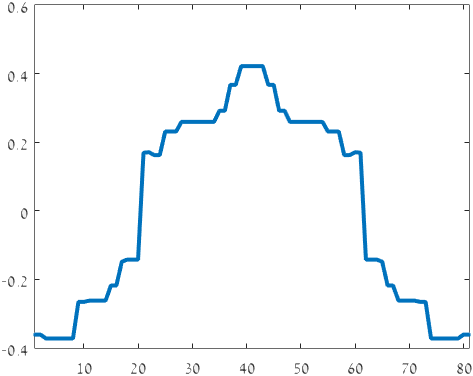}
				}
				\label{fig:tgv_ef_first_inst}
				\subfloat[intermidiate step in proposed method]{
					\includegraphics[width=0.25\textwidth,valign=b]{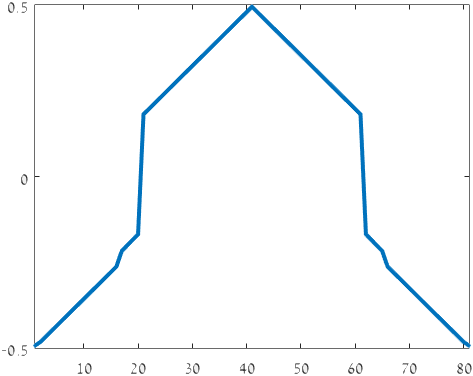}
				}
				\label{fig:tgv_ef_final_inst}
				\subfloat[converged E.F $\lambda_{prop} = 0.0020$]{
					\includegraphics[width=0.25\textwidth,valign=b]{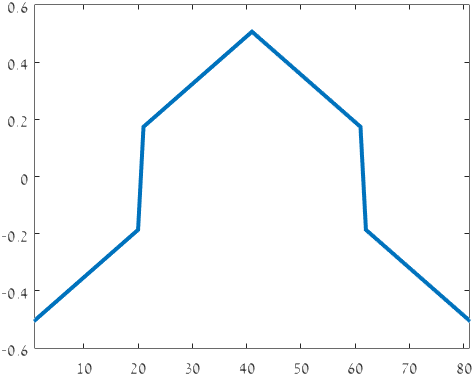}
				}
				\\
				\subfloat[intermediate step in IPM]{
					\includegraphics[width=0.25\textwidth,valign=b]{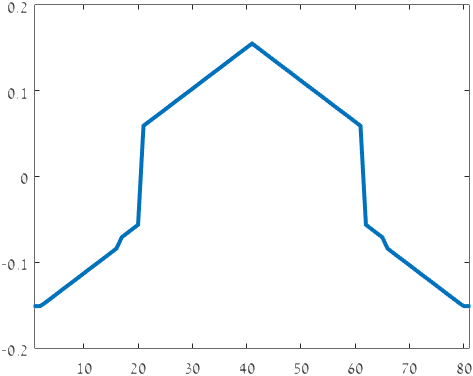}
				}
				\label{fig:tgv_ef_ipm_first_inst}
				\subfloat[intermediate step in IPM]{
					\includegraphics[width=0.25\textwidth,valign=b]{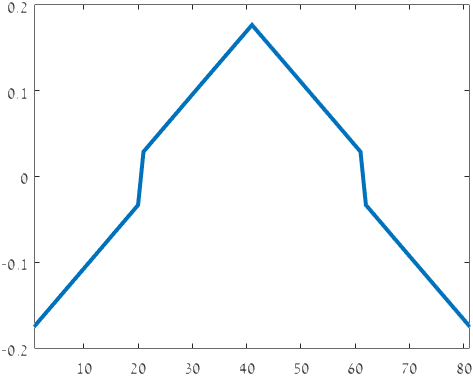}
				}
				\label{fig:tgv_ef_ipm_final_inst}
				\captionsetup{justification=centering}
				\subfloat[converged E.F $\lambda_{IPM} = 0.0019$]{
					\includegraphics[width=0.25\textwidth,valign=b]{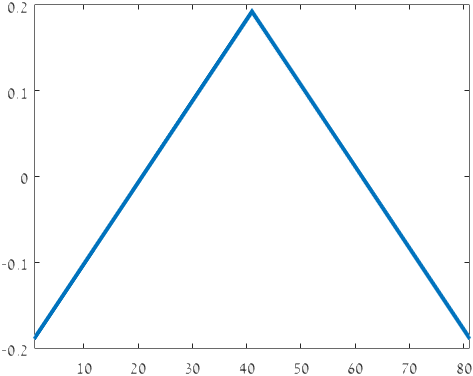}
				}
				
			\end{minipage}
			
			\captionsetup{justification=justified}
			\caption{A 1D example of the two methods for the TGV functional. The upper row represents the proposed method, while the lower row is the IPM. (a) is the initial input. (b-c) \& (e-f) are examples of intermediate steps of the iterative methods. (d) \& (g) shows the final state (i.e the eigenfunction) each method converged to.}
			\label{fig:tgv_ef_noise}
		\end{figure}
		
		\begin{figure}[!htb]
			\centering
			\includegraphics[width=0.95 \textwidth,valign=c]{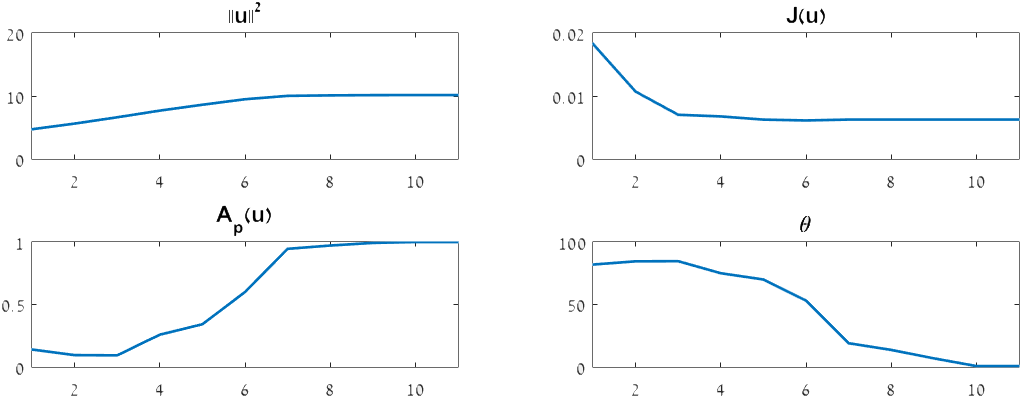}
			\caption{The evolution of $J(u),\, \|u\|^2, \, A_p(u)$ and $\theta$ as a function of $t$, for the given initial input in figure \ref{fig:tgv_ef_noise}. Notice that $\|u\|^2$ is monotonically increasing and that $J(u)$ is monotonically decreasing}
			\label{fig:tgv_ef_noise_outcome}
		\end{figure}

		\begin{figure}[!ht]
			\captionsetup{justification=centering}
			\centering
			\subfloat[input function]{
				\includegraphics[width=0.20\textwidth,valign=c]{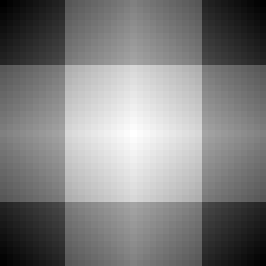}
			}
			\label{fig:tgv_2d_input}
			\begin{minipage}{0.75\textwidth}
				\subfloat[intermidiate step in proposed method]{
					\includegraphics[width=0.25\textwidth,valign=b]{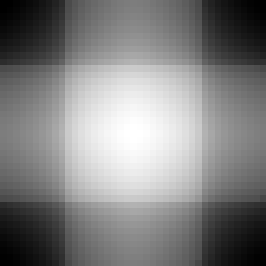}
				}
				\label{fig:tgv_2d_first_inst}
				\subfloat[converged E.F $\lambda_{prop} = 0.049$ ]{
					\includegraphics[width=0.25\textwidth,valign=b]{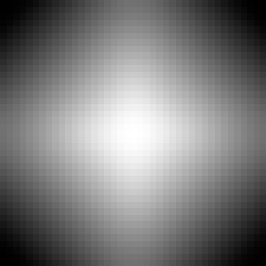}
				}
				\label{fig:tgv_2d_final_inst}
				\subfloat[converged E.F 3D view]{
					\includegraphics[width=0.35\textwidth,valign=b]{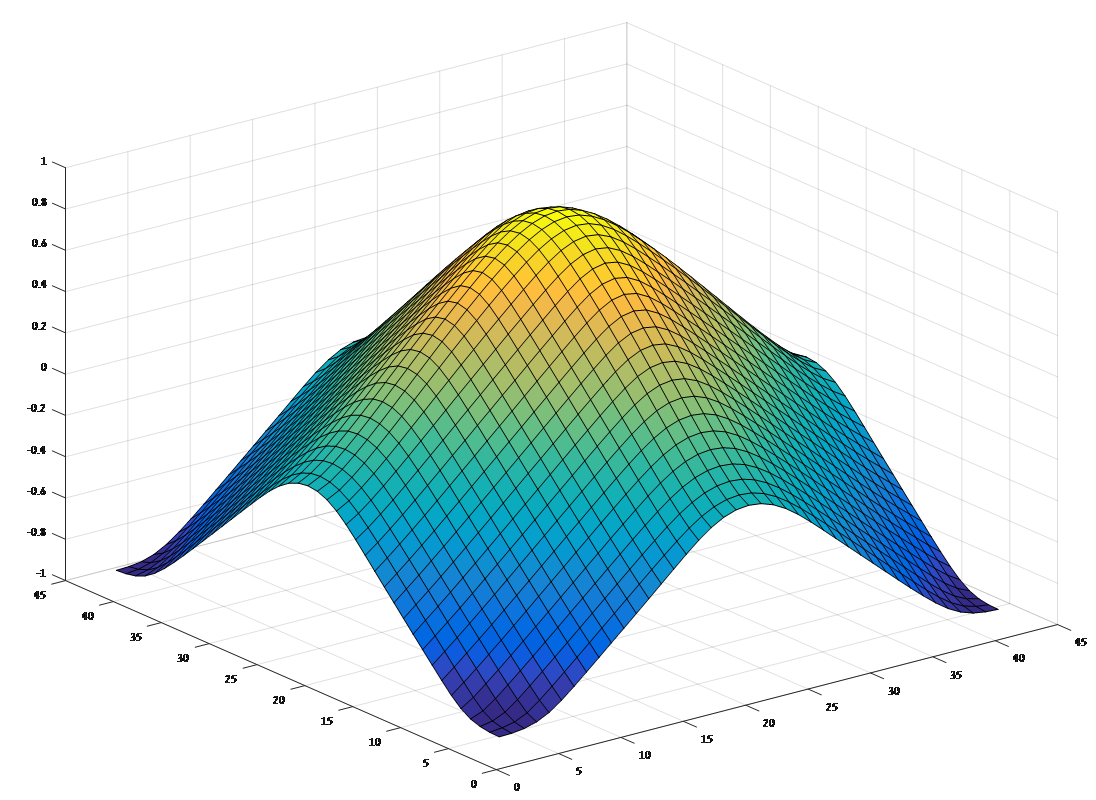}
				}
				\\
				\captionsetup{justification=justified}
				\subfloat[intermediate step in IPM]{
					\includegraphics[width=0.25\textwidth,valign=b]{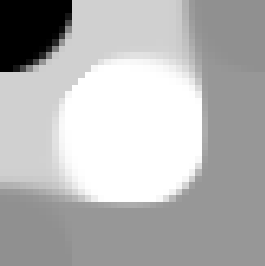}
				}
				\captionsetup{justification=centering}
				\label{fig:tgv_2d_ipm_first_inst}
				\subfloat[converged E.F $\lambda_{IPM} = 0.008$]{
					\includegraphics[width=0.25\textwidth,valign=b]{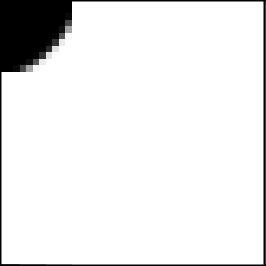}
				}
				\label{fig:tgv_2d_ipm_final_inst}
				\subfloat[converged E.F 3D view]{
					\includegraphics[width=0.35\textwidth,valign=b]{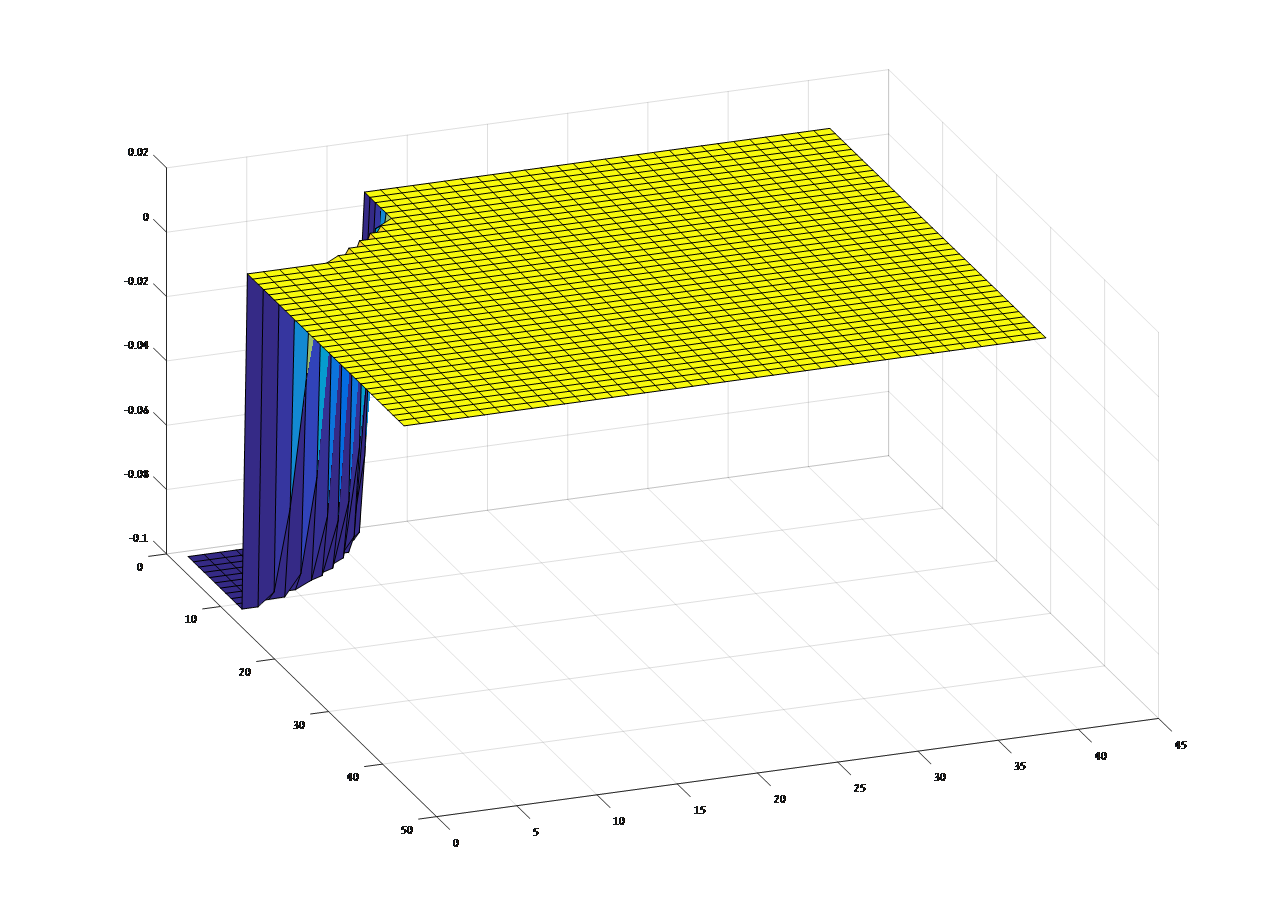}
				}
				
			\end{minipage}
			
			\captionsetup{justification=justified}
			\caption{A 2D example of the two methods for the TGV functional. The upper row represents the proposed method, while the lower row is the IPM. (a) is the initial input. (b) \& (e) are examples of intermediate steps in the iterative methods. (c) \& (f) shows the final state (i.e the eigenfunction) each method converged to.(d) \& (g) are 3D views for better understanding of the shapes of the resulted eigenfunction of each method.}
			\label{fig:tgv_2d}
		\end{figure}
		
		\begin{figure}[!ht]
			\centering
			\includegraphics[width=0.95 \textwidth,valign=c]{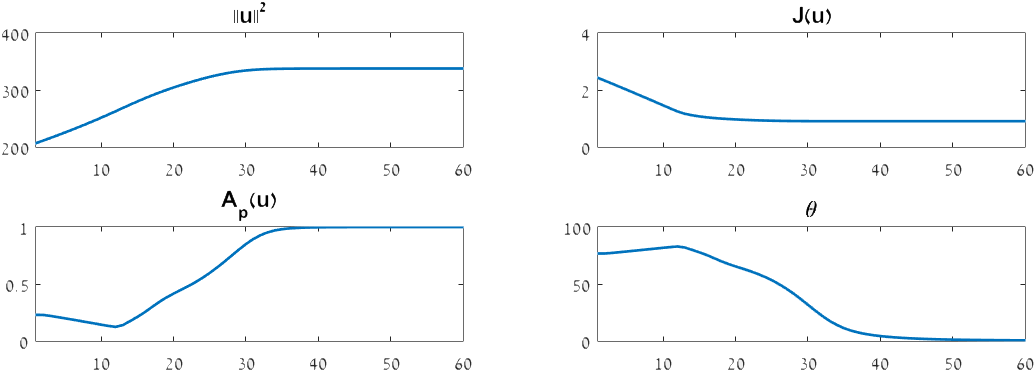}
			\caption{The evolution of $J(u),\, \|u\|^2, \, A_p(u)$ and $\theta$ as a function of $t$, for the given initial input in figure \ref{fig:tgv_2d}. Notice that $\|u\|^2$ is monotonically increasing and that $J(u)$ is monotonically decreasing}
			\label{fig:tgv2d_outcome}
		\end{figure}
		
	\begin{figure}[t]
		\captionsetup{justification=centering}
		\centering
		\subfloat[input function]{
			\includegraphics[width=0.18\textwidth,valign=c]{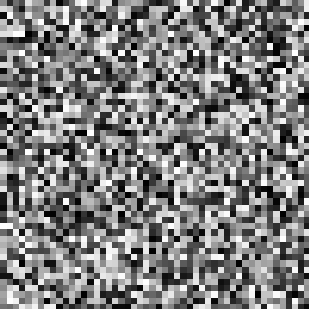}
			
		}
		\quad
		\label{fig:tv_noise_input}
		\subfloat[]{
			\includegraphics[width=0.18\textwidth,valign=c]{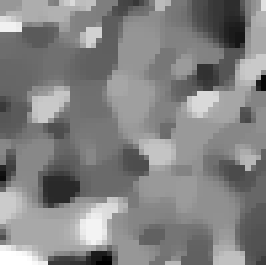}
		}
		\quad
		\label{fig:tv_noise_first_inst}
		\subfloat[]{
			\includegraphics[width=0.18\textwidth,valign=c]{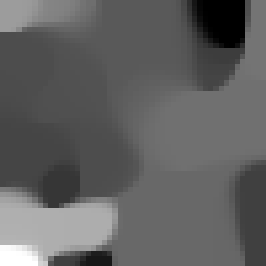}
		}
		\quad
		\label{fig:tv_noise_second_inst}
		\subfloat[converged E.F $\lambda_{prop} = 1.941$]{
			\includegraphics[width=0.18\textwidth,valign=c]{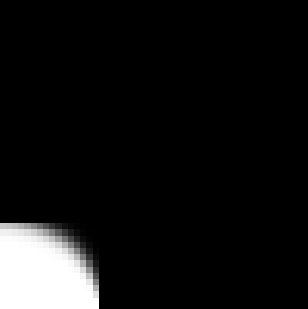}
		}
		\label{fig:tv_noise_final_inst}
		\captionsetup{justification=justified}
		\caption{A 2D example of the two methods for the TV functional. (a) is the initial input that is random Gaussian noise. (b-c) are two samples of intermediate steps of the iterative method. (d) shows the final state (i.e the eigenfunction) the proposed method converged to.}
		\label{fig:tv_noise}
	\end{figure}
	
	\begin{figure}[t]
		\captionsetup{justification=centering}
		\centering
		\subfloat[input function]{
			\includegraphics[width=0.18\textwidth,valign=c]{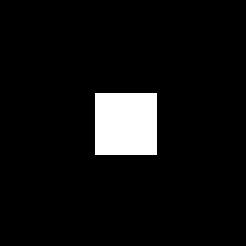}
			
		}
		\quad
		\label{fig:tv_square_input}
		\subfloat[]{
			\includegraphics[width=0.18\textwidth,valign=c]{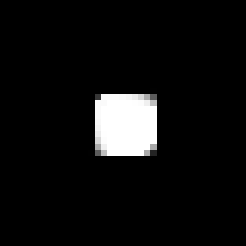}
		}
		\quad
		\label{fig:tv_square_first_inst}
		\subfloat[]{
			\includegraphics[width=0.18\textwidth,valign=c]{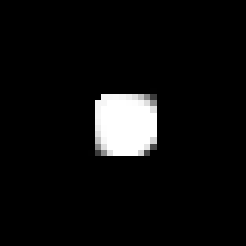}
		}
		\quad
		\label{fig:tv_square_second_inst}
		\subfloat[converged E.F $\lambda_{prop} = 3.835$ ]{
			\includegraphics[width=0.18\textwidth,valign=c]{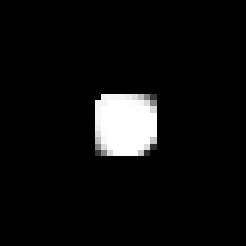}
		}
		\label{fig:tv_square_final_inst}
		\captionsetup{justification=justified}
		\caption{A 2D example of the proposed method for the TV functional that resulted in the eigenfunction used in the example of figure \ref{fig:EF_spec}. (a) is the initial input that is a square. (b-c) are two samples of intermediate steps of the iterative method. (d) shows the final state (i.e the eigenfunction) the proposed method converged to.}
		\label{fig:tv_square}
	\end{figure}
	
\subsection{Expirements}
We compare our flow method with the IPM. We have performed the comparisons both for the TV and TGV functionals as well as for 1D signals and 2D signals. The comparison is done by applying the same initial conditions for each method and performing iterations until the required convergence criterion is met (which is the same for both cases).
Figure \ref{fig:f_point_noise} shows the results for the case of generating an eigenfunction for the TV functional. Both our method and the IPM are depicted. It shows a sample of some iterations and the final result the algorithms converge to. As expected, both methods converge to an eigenfunction. However, the outcome of the algorithms is different. We note that while both methods converge to non-trivial eigenfunctions, the IPM converges to a simpler one, with less structure (closer to the first ground-state \cite{benning2012ground}). A similar phenomenon happens also in figures \ref{fig:tgv_ef_noise} and \ref{fig:tgv_2d}. These figures illustrate the progression of the two methods for the TGV case, in a 1D and 2D setting, respectively. Again in these examples one notices the difference between the two methods. Our method converges to an eigenfunction that is less trivial and is able to give more insight to what shapes the functionals preserve. To the best of our knowledge the result our algorithm converges to in figure \ref{fig:tgv_2d} is a new type of eigenfunction, which is yet to be formalized in an analytical closed form. We would further want to give focus to the time-step parameter $\Delta t$. Although our method often requires more iterations in order to converge to an eigenfunction (even though this is not always the case, e.g figure \ref{fig:tgv_ef_noise}) the number of iterations is dependent on $\Delta t$. Increasing $\Delta t$ will result in less iterations needed for convergence, but will increase the probability that the outcome will be of a more trivial state. Thus, if desired, one can incorporate an adaptive scheme in order to reduce the amount of iterations needed, while being able to maintain convergence to complex eigenfunctions. We can conclude that one of the advantages of our method is the great flexibility of tuning the time-step during the progression of the process.

For each of the examples in figures \ref{fig:f_point_noise}, \ref{fig:tgv_ef_noise}, and \ref{fig:tgv_2d} we also illustrated in figures \ref{fig:tv_point_outcome}, \ref{fig:tgv_ef_noise_outcome}, and \ref{fig:tgv2d_outcome} how $\|u\|^2,\, J(u), \, A_p(u)$ and $\theta$ change throughout the process. As given in theorem \ref{th:for_flow} we can see that $\|u\|^2$ is monotonically increasing and that $J(u)$ is monotonically decreasing. Note however that for $A_p(u)$ and $\theta$ there is no consistent behavior. Figures \ref{fig:tv_square} and \ref{fig:tv_noise} show further examples of our method. In these cases both methods reach very similar results. Figure \ref{fig:tv_square} demonstrates how the flow converged to the eigenfunction used in the example of figure \ref{fig:EF_spec}, and figure \ref{fig:tv_noise} depicts an example of the convergence of our method when given gaussian noise as the initial condition.

We also show a few results for the inverse flow and for the linear extension to the proposed flow. For the implementation of both flows we use the explicit scheme.
In the case of the inverse flow we performed the evolution using the TV functional. We show in figure \ref{fig:tv_for_inv} the resulting difference between the forward flow and the inverse flow given the same initial input. As expected the eigenvalue for the found eigenfunction using the inverse flow is greater than the eigenvalue for the eigenfunction found using the forward flow. In figure \ref{fig:linear_flow} is illustrated the result of the forward flow for a linear operator. We remind that in order to hold the properties given in proposition \ref{th:linear_flow} the operator should be positive semi definite. Therefore we demonstrate this extension on the $-\Delta$ operator. We can see that the flow is converges to an eigenfunction.

\begin{figure}[t]
	\captionsetup{justification=centering}
	\centering
	\subfloat[input function]{
		\includegraphics[width=0.20\textwidth,valign=c]{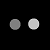}
	}
	\label{fig:tv_for_input}
	\begin{minipage}{0.75\textwidth}
		\subfloat[intermediate step in forward flow]{
			\includegraphics[width=0.25\textwidth,valign=b]{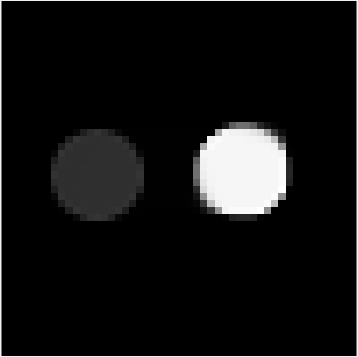}
		}
		\label{fig:tv_for_first_inst}
		\subfloat[intermediate step in forward flow]{
			\includegraphics[width=0.25\textwidth,valign=b]{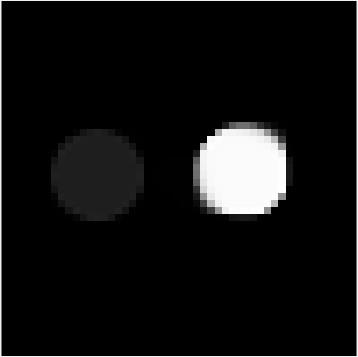}
		}
		\label{fig:tgv_for_second_inst}
		\subfloat[converged E.F $\lambda_{for} = 1.741$]{
			\includegraphics[width=0.25\textwidth,valign=b]{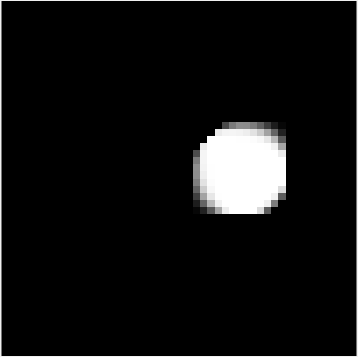}
		}
		\\
		\captionsetup{justification=centering}
		\subfloat[intermediate step in inverse flow]{
			\includegraphics[width=0.25\textwidth,valign=b]{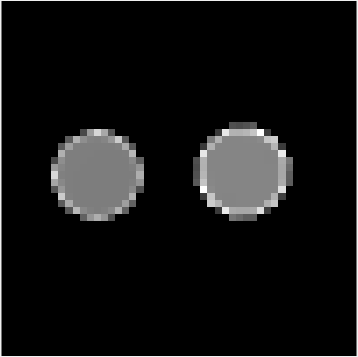}
		}
		\captionsetup{justification=centering}
		\label{fig:tv_inv_first_inst}
		\subfloat[intermediate step in inverse flow]{
			\includegraphics[width=0.25\textwidth,valign=b]{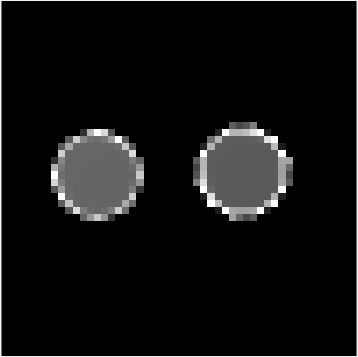}
		}
		\label{fig:tv_inv_final_inst}
		\subfloat[converged E.F $\lambda_{inv} = 4.997$]{
			\includegraphics[width=0.25\textwidth,valign=b]{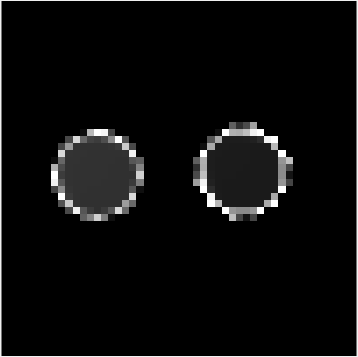}
		}
		
	\end{minipage}
	
	\captionsetup{justification=justified}
	\caption[A 2D example of the proposed forward flow and the proposed inverse flow] {A 2D example of the two forward and inverse flow for the TV functional. The upper row represents the forward flow, while the lower row is the inverse flow. (a) is the initial input. (b,c,e,f) \& (e) are examples of intermediate steps in the iterative methods. (d) \& (g) shows the final state (i.e the eigenfunction) each flow converged to.}
	\label{fig:tv_for_inv}
\end{figure}

\begin{figure}[t]
	\captionsetup{justification=centering}
	\centering
	\subfloat[input function]{
		\includegraphics[width=0.20\textwidth,valign=c]{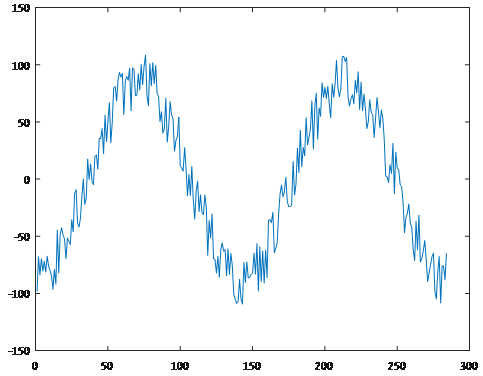}
		
	}
	\quad
	\label{fig:lap_input}
	\subfloat[]{
		\includegraphics[width=0.18\textwidth,valign=c]{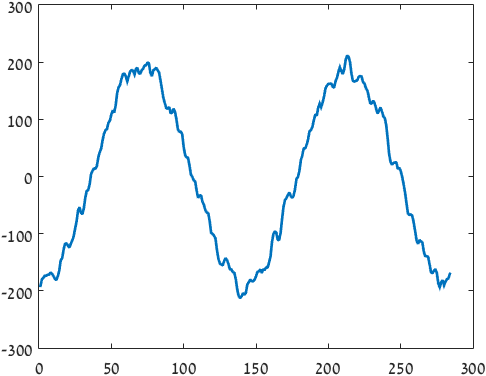}
	}
	\quad
	\label{fig:lap_first_inst}
	\subfloat[]{
		\includegraphics[width=0.18\textwidth,valign=c]{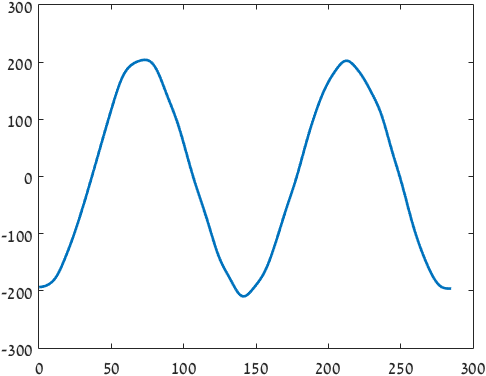}
	}
	\quad
	\label{fig:lap_second_inst}
	\subfloat[converged E.F $\lambda_{prop} = 0.002$ ]{
		\includegraphics[width=0.18\textwidth,valign=c]{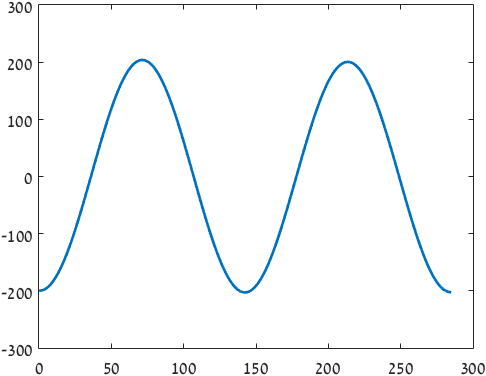}
	}
	\label{fig:lap_final_inst}
	\captionsetup{justification=justified}
	\caption[A 2D example of the proposed flow for a linear operator] {A 1D example of the proposed method for the Laplacian operator. (a) is the initial input. (b-c) are two samples of intermediate steps of the iterative method. (d) shows the final state (i.e the eigenfunction) the proposed method converged to.}
	\label{fig:linear_flow}
\end{figure}
	
	\section{Conclusion}
	In this paper we presented a new method for generating eigenfunctions induced by nonlinear one-homogeneous functionals. In particular we have exemplified our method on the TV and TGV functionals in the 1D and 2D settings, showing numerical convergence to non-trivial eigenfunctions. The flow is interpreted numerically as a series of convex optimization problems and is solved by a primal-dual algorithm \cite{Chambolle2011}.
We further introduced a new measure of affinity that indicates how close a function is to being an eigenfunction of some operator.
	
	Future directions for work include a deeper investigation of the properties of the inverse flow given in \eqref{eq:inv_flow} and for the linear case given in \eqref{eq:liner_for_flow}.
Flows based on nonlinear operators $T$, which are not based on subgradients of functionals, such as Weickert's anisotropic diffusion \cite{wk_book} operator will also be
examined. Further subjects of investigation are extensions of the proposed method to graphs and suitable nonlinear operators as the graph $p-$Laplacian, finding Cheeger sets and more.
	
	\section*{Acknowledgments}
	We would like to acknowledge support by the Israel Science Foundation (grant No. 718/15).
	
	\bibliographystyle{plain}
	\bibliography{references}
	
\end{document}